\documentclass{article}

\usepackage[main, final]{arxiv}

\usepackage[utf8]{inputenc}
\usepackage[T1]{fontenc}
\usepackage{hyperref}
\usepackage{url}
\usepackage{booktabs}
\usepackage{amsfonts}
\usepackage{amsmath}
\usepackage{amssymb}
\usepackage{amsthm}
\usepackage{nicefrac}

\theoremstyle{plain}
\newtheorem{lemma}{Lemma}

\theoremstyle{remark}

\usepackage{microtype}
\usepackage{xcolor}
\usepackage{multirow}
\usepackage{makecell}
\usepackage{graphicx}
\usepackage{subcaption}
\usepackage{algorithm}
\usepackage{algorithmic}
\usepackage{tikz}
\usepackage{comment}
\usetikzlibrary{arrows.meta,positioning,calc,decorations.pathmorphing,decorations.markings,shapes.geometric,shapes.symbols,backgrounds,fit}
\title{One Pass Is Not Enough: Recursive Latent Refinement for Generative Models}

\author{%
  Mehdi Esmaeilzadeh\textsuperscript{1} \quad
  Alexia Jolicoeur-Martineau\textsuperscript{2} \quad
  Chirag Vashist\textsuperscript{1} \quad
  Ke Li\textsuperscript{1} \\[0.5em]
  \textsuperscript{1}Simon Fraser University \quad \textsuperscript{2}Independent
}

\begin{document}

\maketitle

\begin{figure}[h]
  \centering
  \includegraphics[width=\linewidth]{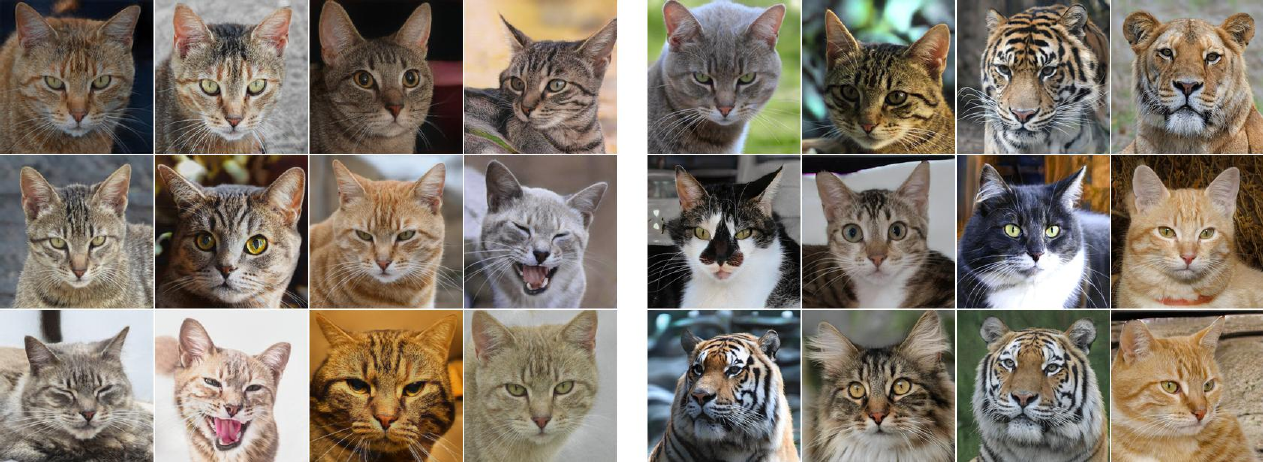}
  \caption{Unconditional AFHQ-v1 ($512{\times}512$) samples from StyleGAN2-ADA without RTM (left) vs.\ with RTM (right). RTM improves both quality (FID $4.79$ vs.\ $4.99$) and diversity (Recall $0.565$ vs.\ $0.507$).}
  \label{fig:hero}
\end{figure}


\begin{abstract}
Despite remarkable progress, image generation is far from solved. The 
dominant metric, FID, conflates sample fidelity with mode coverage and 
is close to being saturated. Yet a model can still exhibit mode collapse 
while achieving a low FID, since a handful of sharp, near-duplicate images can outscore a model that faithfully covers the full data 
distribution. We argue that precision and recall are essential 
complements to FID, and that because FID is already saturated, the more 
meaningful goal is to improve diversity and coverage. Achieving high 
recall requires a model that explicitly prioritizes mode coverage, 
unlike most generative models, which optimize sample fidelity. We 
introduce RTM, which replaces the single-pass latent mapping in 
style-based generators with an iterative refinement process, and show 
that this consistently improves both quality and diversity. Integrated 
with Implicit Maximum Likelihood Estimation (IMLE), which optimizes 
mode coverage by design, RTM achieves the highest precision and recall 
among current state-of-the-art approaches while maintaining competitive 
FID, with improvements across CIFAR-10, CelebA-HQ at $256{\times}256$, 
and nine few-shot benchmarks. RTM also improves StyleGAN2 and 
StyleGAN2-ADA on CIFAR-10 and AFHQ-v1 at $512{\times}512$, 
demonstrating that the benefit is not specific to IMLE. Unlike 
flow-matching baselines that achieve competitive FID at the expense of 
coverage, recursive refinement improves both quality and diversity 
simultaneously.
\end{abstract}

\section{Introduction}
Research on generative models has made remarkable advances and produced a rich family of methods, including diffusion models~\citep{ho2020ddpm,karras2022edm}, flow matching~\citep{lipman2023fm, liu2022flow}, and generative adversarial networks (GANs)~\citep{goodfellow2014gan,karras2020stylegan2}. On image generation tasks, Fréchet inception distance (FID)~\citep{heusel2017fid} has become the standard evaluation metric, and with each successive paper, the state-of-the-art FID has fallen to the low single digits and is close to being saturated. Does this mean that image generation is now solved? 



We argue that image generation is still far from being solved, and the remarkable progress in improving FID has obscured other important challenges. A well-documented limitation of FID~\citep{kynkaanniemi2019pr,naeem2020prdc} is that it conflates sample fidelity with mode coverage into a single scalar, making it impossible to distinguish a model that generates realistic but low-diversity samples from one that faithfully covers the full data distribution. 

A model that produces
a handful of sharp, near-duplicate images per class can therefore attain
a lower FID than one that faithfully covers the long tail. In other words, a model can attain a low FID even if it exhibits \emph{mode collapse}, which is a long-standing problem with generative models where they concentrate on a few high-fidelity modes and drop
others~\citep{salimans2016is,lucic2018gansequal,arora2017gen,goodfellow2016tutorial}.
Distilled and few-step diffusion models show a similar
erosion in the number of modes covered as their step count
shrinks~\citep{salimans2022pd,yin2024dmd,sehwag2022lowdensity}. 


Precision and Recall~\citep{kynkaanniemi2019pr,naeem2020prdc} directly address this by measuring fidelity and coverage independently. Precision measures the fraction of generated samples that are realistic, and Recall measures the fraction of the real data distribution that the generator covers. Unlike FID, where mode collapse can be masked by sharp, concentrated samples, a drop in Recall makes coverage failure immediately visible.

This view informs our goal and choice of methodology. Because FID is already close to being saturated, our goal is not necessarily to push FID even lower. Instead, we aim to improve Precision and Recall, while still maintaining a reasonably low FID. Improving Recall in particular requires a training objective that explicitly targets mode coverage; most generative models optimize sample fidelity instead. Our method is based on Implicit Maximum Likelihood Estimation (IMLE), which satisfies this requirement by construction: it guarantees every training image a nearby generated sample, making mode collapse impossible by design. RTM-IMLE achieves the highest Precision and Recall among current state-of-the-art approaches while maintaining competitive FID.
Figure~\ref{fig:trilemma} places IMLE in the broader
landscape of generative models along the quality, diversity, and speed axes.

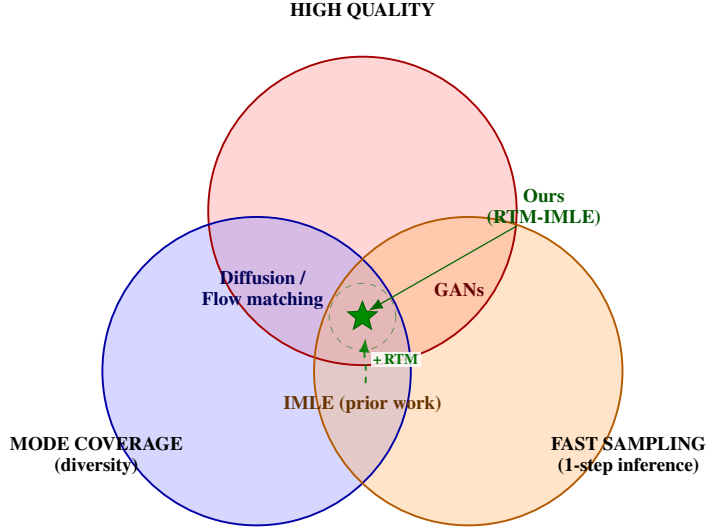
\begin{figure}[t]
  \centering
  \scalebox{0.80}{%
  \begin{tikzpicture}[
      font=\small,
      >={Latex[length=2.4mm]},
      axis/.style={font=\bfseries\small, align=center},
      fam/.style={align=center, font=\small\bfseries, inner sep=1.5pt},
      ourstar/.style={star, star points=5, star point ratio=2.3,
                      draw=green!40!black, line width=0.6pt,
                      fill=green!55!black, inner sep=2.2pt},
    ]
    \def\radius{2.55cm}
    \coordinate (Cq) at (0,      1.60);
    \coordinate (Cd) at (-1.75, -1.05);
    \coordinate (Cf) at ( 1.75, -1.05);

    \fill[red!45,   fill opacity=0.35] (Cq) circle[radius=\radius];
    \fill[blue!45,  fill opacity=0.35] (Cd) circle[radius=\radius];
    \fill[orange!60,fill opacity=0.35] (Cf) circle[radius=\radius];
    \draw[red!65!black,   line width=0.9pt] (Cq) circle[radius=\radius];
    \draw[blue!65!black,  line width=0.9pt] (Cd) circle[radius=\radius];
    \draw[orange!70!black,line width=0.9pt] (Cf) circle[radius=\radius];

    \node[axis] at ( 0.00,  4.90) {HIGH QUALITY};
    \node[axis] at (-4.40, -2.45) {MODE COVERAGE\\(diversity)};
    \node[axis] at ( 4.40, -2.45) {FAST SAMPLING\\(1-step inference)};

    \node[fam, text=red!40!black]  at ( 1.60,  0.30) {GANs};
    \node[fam, text=blue!35!black] at (-1.65,  0.30) {Diffusion /\\Flow matching};
    \node[fam, text=orange!40!black] at ( 0.00, -1.60) {IMLE (prior work)};

    \node[ourstar] (ours) at (0, -0.15) {};
    \node[fam, text=green!35!black] at (3.00, 1.65) {Ours\\(RTM-IMLE)};
    \draw[->, line width=0.6pt, green!40!black]
         (2.55, 1.35) -- (ours.north east);

    \draw[->, dashed, line width=0.9pt, green!50!black, opacity=0.90]
        (0.05, -1.25) to[out=85, in=265] (0.05, -0.50);
    \node[font=\scriptsize\bfseries, text=green!45!black,
          fill=white, fill opacity=0.85, inner sep=1pt, text opacity=1]
        at (0.55, -0.85) {+\,RTM};

    \draw[green!45!black, line width=0.4pt, dashed, opacity=0.55]
        (0, -0.15) circle[radius=0.55cm];
  \end{tikzpicture}%
  }
  \caption{ Each circle marks one of quality, diversity, or fast (1-step) sampling; families sit at the intersections. RTM-IMLE pushes IMLE into the triple intersection. Adapted from~\citet{xiao2022ddgan}.}
  \label{fig:trilemma}
\end{figure}

IMLE works by minimizing the distance between each data samples to its nearest generated sample, where the nearness may be defined in raw data space or latent space. 
It resists mode collapse by construction, since every training image is guaranteed a nearby generated image, but it has
historically lagged behind in sample quality. Sample quality is limited not only by the training objective but also by the generator's architecture responsible for
mapping noise to images. IMLE-based methods have generally relied on the StyleGAN family of architectures for their generators~\citep{li2018imle,aghabozorgi2023adaimle,vashist2024rsimle}. The StyleGAN architecture ~\citep{karras2019stylegan,karras2020stylegan2,karras2020ada}
consists of a small mapping network that turns Gaussian noise
 $z \in \mathbb{R}^d$ into a style code
$w \in \mathbb{R}^{d'}$, and a convolutional decoder that conditions on
$w$ via Adaptive Instance Normalization~\citep{huang2017adain} to
progressively upsamples a learned constant feature map into an image. The mapping network is a multilayer perceptron (MLP) network (generally with 8 layers) processed in a single forward pass. 

In every prior IMLE and StyleGAN model, the mapping network is a plain MLP processed in a single forward pass~\citep{karras2019stylegan,karras2020stylegan2}. This forces the mapper to determine every aspect of the style code simultaneously, identity, structure, texture, and fine detail, in one shot. Because the decoder is highly sensitive to small variations in $w$~\citep{karras2019stylegan}, any inaccuracy in $w$ at this stage produces visible artifacts in the final image. A natural response is to make the MLP deeper or wider, but this does not change the fundamental structure: a feed-forward chain still determines $w$ in a single pass, with no mechanism to revise an earlier decision in light of later computation. Our central insight is that refinement constitutes a qualitatively different computation: the mapper produces a coarse estimate of $w$ first and progressively corrects it over multiple cycles, with early cycles establishing coarse structure such as identity, composition, and pose, and later cycles refining texture, sharpening, and color.


To this end, we propose the Recursive Token Mapper (RTM), a drop-in replacement for the single-pass MLP in the StyleGAN architecture. RTM adapts the Tiny Recursive Model of~\citet{jolicoeurmartineau2025trm} to the generative setting, refining latent tokens through nested recursive cycles to gain effective depth through recursion rather than width; the full architecture is described in Section~\ref{sec:generator}.

In summary, our contributions are: \textbf{(i)} the Recursive Token
Mapper, a drop-in recursive replacement for the MLP mapping network
shared by the StyleGAN family of generators; \textbf{(ii)} integrated
with RS-IMLE~\citep{vashist2024rsimle}, RTM improves FID, precision,
and recall over a vanilla RS-IMLE baseline on nine few-shot
benchmarks, unconditional CIFAR-10, and CelebA-HQ at
$256{\times}256$, while retaining IMLE's direct latent-to-image map
and one-step generation; and \textbf{(iii)} integrated with the
adversarially-trained StyleGAN2~\citep{karras2020stylegan2} and
StyleGAN2-ADA~\citep{karras2020ada} generators, RTM lowers FID and
raises precision, density, and coverage on unconditional CIFAR-10 and
AFHQ-v1 relative to the corresponding non-recursive baselines, showing that the
benefit is not confined to the context of IMLE training.

\section{Related Work}
\label{sec:related}

\subsection{Diffusion and flow-matching}
\label{sec:related_diffusion}
The strongest current baselines in terms of FID belong to the diffusion
and flow-matching families: DDPM~\citep{ho2020ddpm},
EDM~\citep{karras2022edm}, LSGM~\citep{vahdat2021lsgm}, Flow Matching
(FM)~\citep{lipman2023fm}, OT Flow
Matching~\citep{tong2023otcfm}, Mean
Flows~\citep{geng2025meanflow}, and Inductive Moment
Matching~\citep{zhou2025imm}. These methods learn a time-dependent
vector field $u_\theta(x,t)$ that defines an ordinary or stochastic
differential equation taking Gaussian noise at $t{=}0$ to a data
sample at $t{=}1$. 

Generation from a noise code $z$ is therefore not
a function $G_\theta(z)$, but the trajectory of an iterative solver
initialised at $z$, which has three consequences. First, the
correspondence between a noise code and an image is only defined
implicitly through the solver; the training objective does not pull a
particular latent towards a specific training image, but matches
expected vector fields. Second, reaching a low-FID sample typically
requires tens to hundreds of neural function evaluations; distilled
one-step
variants~\citep{salimans2022pd,yin2024dmd,song2023consistency,geng2025meanflow,zhou2025imm}
accelerate inference but consistently trade away coverage in the
process: \citet{yin2024dmd} explicitly observe a drop in sample
diversity when collapsing a multi-step diffusion teacher into a
single-step student, and \citet{salimans2022pd} report a similar
quality--diversity erosion as the number of sampling steps shrinks.
Third, even at full step count diffusion solvers systematically
under-cover low-density regions of the data
distribution~\citep{sehwag2022lowdensity}.

\subsection{Generative adversarial networks}
\label{sec:related_gan}
GANs~\citep{goodfellow2014gan,karras2019stylegan,karras2020stylegan2,karras2020ada,sauer2022styleganxl}
sit at the opposite end of the same trade-off: they generate in a
single forward pass, but a long line of work, starting
with~\citet{salimans2016is,goodfellow2016tutorial,arora2017gen} and
quantified at scale by~\citet{lucic2018gansequal}, has established
that adversarial training is prone to mode collapse, where the generator concentrates on a few high-fidelity modes and silently drops
others. Considerable effort has gone into mitigating this failure mode,
including training stabilizers and
regularisers~\citep{salimans2016is,karras2020stylegan2}, data
augmentation strategies~\citep{karras2020ada}, and hybrid GAN/diffusion
objectives~\citep{xiao2022ddgan}, but the failure mode remains the
central concern with adversarial training. Style-based GAN
families~\citep{karras2019stylegan,karras2020stylegan2,karras2020ada}
introduced the two-component design of a mapping network followed by a
convolutional decoder that our RTM builds on.

\subsection{IMLE-family generators}
\label{sec:related_imle}
Implicit Maximum Likelihood Estimation (IMLE)~\citep{li2018imle} is a
direct-mapping alternative that trains a generator $G_\theta(z)$ with
the explicit guarantee that every training image $x_i$ has some latent
$z_i$ whose generated image $G_\theta(z_i)$ is close to it in a
learned feature space. A pool of random latents is sampled, and the nearest
generated sample to each $x_i$ is selected, and the generator is
pulled towards the matched training image. This one-to-one assignment is what makes IMLE robust to the mode-collapse failure modes documented
for adversarial~\citep{lucic2018gansequal} and
distilled-diffusion~\citep{yin2024dmd,salimans2022pd} generators: every real image is guaranteed a nearby preimage by construction.
Inference is a single forward pass, and the latent-to-image map is an
ordinary neural network. Adaptive IMLE
(AdaIMLE)~\citep{aghabozorgi2023adaimle} introduced adaptive per-image
thresholds during training, and the more recent Rejection-Sampling IMLE
(RS-IMLE)~\citep{vashist2024rsimle} closes the train and test prior
gap by rejecting pool latents whose generated images are too close to
existing training images. Across this entire line of work, the mapping network has remained an eight-layer MLP inherited from StyleGAN.

\subsection{Recursive and iterative architectures}
\label{sec:related_recursive}
Recursive computation has a long history in deep learning, from
RNN-style weight-tying to Universal
Transformers~\citep{dehghani2019universal} and recent
looped-transformer analyses~\citep{giannou2023looped}. Closer to our
setting, the Hierarchical Reasoning Model
(HRM)~\citep{wang2025hrm} and its compact successor, the Tiny Recursive
Model (TRM)~\citep{jolicoeurmartineau2025trm} introduce nested
$H{\times}L$ recursive cycles around a single shared block, paired
with deep supervision and a learned halting head, on discrete reasoning
benchmarks. We adapt this recursive architecture to the generative
setting as a mapping network for image generation.

\section{Method}
\label{sec:method}

\subsection{Background: IMLE and RS-IMLE}
\label{sec:imle}

IMLE~\citep{li2018imle} trains a generator $G_\theta$ on a dataset $\{x_1,\dots,x_n\}$ by guaranteeing that every training image is paired with a noise vector that maps near it. At each round, a pool of $m \gg n$ candidate latents $\{\tilde z_j\}_{j=1}^m$ is drawn from a Gaussian prior $p(z)$, decoded into images $G_\theta(\tilde z_j)$, and matched to training images by a nearest-neighbour search in a learned feature space $\phi$. The matched latents are then used as training inputs:
\begin{equation}
  \sigma(i) = \arg\min_{j \in \{1,\dots,m\}} \big\| \phi(x_i) - \phi(G_\theta(\tilde z_j)) \big\|_2,
  \qquad
  \min_\theta \; \sum_{i=1}^{n} \mathcal{L}\!\left(G_\theta(\tilde z_{\sigma(i)}),\, x_i\right),
  \label{eq:imle}
\end{equation}
where $\mathcal{L}$ combines an LPIPS~\citep{zhang2018lpips} perceptual term and a pixel-level reconstruction term. Pairing every training image with its own latent makes mode collapse impossible by construction, and inference is a single forward pass.

RS-IMLE~\citep{vashist2024rsimle} closes the gap between the matched-latent training distribution and the i.i.d.\ Gaussian inference distribution by rejecting any pool sample whose generated image is closer than a threshold $\varepsilon$ to some training image, so the generator only learns from latents that look like genuine prior draws.

\subsection{Improved mapping network: the Recursive Token Mapper (RTM)}
\label{sec:generator}

Following StyleGAN~\citep{karras2019stylegan}, the IMLE generator factors into two components. A small mapping network maps a noise vector $z \sim \mathcal{N}(0, I_d)$ to a style vector $w \in \mathbb{R}^d$, and a convolutional decoder conditions on $w$ via Adaptive Instance Normalization~\citep{huang2017adain} and progressively upsamples a learned constant feature map into an image. The decoder architecture follows each baseline: residual blocks with $1{\times}1$ and $3{\times}3$ convolutions, GELU activations and noise injection for the few-shot runs; ConvNeXt-style blocks~\citep{liu2022convnext} for CIFAR-10 at $32 \times 32$ and CelebA-HQ at $256{\times}256$; and the StyleGAN2 / StyleGAN2-ADA~\citep{karras2020stylegan2,karras2020ada} convolutional decoder for the adversarial-training experiments.

The mapping network is the single component we change. In all prior IMLE work, and in the StyleGAN family the architecture is borrowed from, this network is an MLP processed in a single forward pass: eight layers in the RS-IMLE experiments and two layers in the StyleGAN2 experiments. The decoder is highly sensitive to small variations in $w$, so the mapper's placement of $w$ is a sample-quality bottleneck. We replace it with the Recursive Token Mapper (RTM) introduced in the next subsection, which adapts the Tiny Recursive Model of~\citet{jolicoeurmartineau2025trm} to the generative setting.

\begin{figure}[t]
  \centering
  \begin{tikzpicture}[
      font=\small,
      >={Latex[length=2.0mm]},
      labelbig/.style={font=\bfseries, align=center},
      box/.style={draw, rounded corners=2pt, minimum width=2.4cm,
                  minimum height=0.50cm, align=center, line width=0.55pt,
                  font=\footnotesize, inner ysep=2pt},
      fc/.style={box, fill=blue!10, draw=blue!60!black},
      norm/.style={box, fill=gray!15, draw=gray!60!black, font=\footnotesize},
      proj/.style={box, fill=violet!12, draw=violet!60!black, font=\footnotesize},
      grid/.style={draw, rounded corners=1pt, fill=orange!18,
                   draw=orange!70!black, minimum size=0.52cm,
                   inner sep=0pt, font=\tiny},
      hlbox/.style={draw, dashed, rounded corners=3pt, draw=green!45!black,
                    fill=green!6, line width=0.7pt, inner sep=5pt},
      flbox/.style={draw, rounded corners=2pt, line width=0.55pt,
                    align=center, font=\scriptsize, inner ysep=2pt,
                    minimum height=0.85cm},
      outc/.style={draw, cloud, cloud puffs=11, cloud puff arc=120,
                   aspect=2.2, line width=0.5pt,
                   fill=yellow!12, draw=black!55,
                   minimum width=1.55cm, minimum height=0.80cm,
                   font=\footnotesize, inner sep=1pt},
      arrow/.style={->, line width=0.6pt, black!75},
      narrow/.style={->, line width=0.7pt, black!80},
    ]

    \begin{scope}[local bounding box=LEFT]
      \node[labelbig] at (0, 7.40) {Baseline StyleGAN mapper};

      \def\zTopL{6.85}
      \def\pnYL{6.00}
      \def\fcStart{5.05}
      \def\fcStep{0.95}

      \node (zL)       at (0, \zTopL) {$z \in \mathbb{R}^{d}$};
      \node[norm] (pnL) at (0, \pnYL) {PixelNorm};
      \draw[arrow] (zL) -- (pnL);

      \foreach \i [count=\y from 0] in {1,...,8} {
        \node[fc] (fc\i) at (0, {\fcStart - \y*\fcStep}) {FC$_{\i}$ + LReLU};
      }
      \draw[arrow] (pnL) -- (fc1);
      \foreach \i [count=\j from 2] in {1,2,3,4,5,6,7} {
        \draw[arrow] (fc\i) -- (fc\j);
      }

      \node[outc] (wL) at (0, {\fcStart - 8*\fcStep - 0.75})
            {$w \in \mathbb{R}^{d}$};
      \draw[arrow] (fc8) -- (wL);
    \end{scope}

    \begin{scope}[xshift=6.10cm, local bounding box=RIGHT]
      \node[labelbig] at (0, 7.40) {Our Recursive Token Mapper (RTM)};

      \def\zTopR{6.85}
      \def\pnYR{6.00}
      \def\projY{5.05}
      \def\tokY{4.15}
      \def\hlY{1.05}
      \def\hlH{4.65}
      \def\hlW{6.40}
      \def\readY{-2.05}
      \def\wY{-3.30}

      \node (zR)        at (0, \zTopR) {$z \in \mathbb{R}^{d}$};
      \node[norm] (pnR) at (0, \pnYR)  {PixelNorm};
      \node[proj] (p2s) at (0, \projY) {$z \mapsto$ tokens};
      \draw[arrow] (zR)  -- (pnR);
      \draw[arrow] (pnR) -- (p2s);

      \foreach \i/\x in {0/-0.56, 1/0.0, 2/0.56} {
        \node[grid] (tk\i) at (\x, \tokY) {};
      }
      \node[font=\footnotesize, align=center, text=black!70]
          at (1.40, \tokY) {$Z_0$};
      \draw[arrow] (p2s) -- (tk1);

      \node[hlbox, minimum width=\hlW cm, minimum height=\hlH cm]
          (hl) at (0, \hlY) {};

      \draw[arrow] (tk1.south) -- (hl.north);

      \def\blkW{2.55}
      \def\blkHt{3.20}
      \def\subW{2.30}
      \pgfmathsetmacro{\blkWest}{-\blkW/2}
      \pgfmathsetmacro{\blkEast}{ \blkW/2}
      \pgfmathsetmacro{\blkTop}{\hlY + \blkHt/2}
      \pgfmathsetmacro{\blkBot}{\hlY - \blkHt/2}

      \node[draw, rounded corners=2.5pt, line width=0.65pt,
            fill=green!7, draw=green!50!black,
            minimum width=\blkW cm, minimum height=\blkHt cm]
          (blk) at (0, \hlY) {};

      \node[draw, rounded corners=1.6pt, line width=0.5pt,
            fill=green!22, draw=green!55!black,
            minimum width=\subW cm, minimum height=0.62cm,
            align=center, font=\scriptsize\itshape, text=black, inner ysep=2pt]
          (sl1) at (0, \hlY + 1.20)
          {Token-mix MLP\\
           {\tiny\upshape (or Self-Attn)}};
      \node[draw, rounded corners=1.6pt, line width=0.5pt,
            fill=orange!25, draw=orange!65!black,
            minimum width=\subW cm, minimum height=0.40cm,
            align=center, font=\scriptsize, text=black]
          (an1) at (0, \hlY + 0.29) {Add\,$\oplus$\,RMSNorm};
      \node[draw, rounded corners=1.6pt, line width=0.5pt,
            fill=green!22, draw=green!55!black,
            minimum width=\subW cm, minimum height=0.40cm,
            align=center, font=\scriptsize\itshape, text=black]
          (sl2) at (0, \hlY - 0.51) {Channel-mix MLP};
      \node[draw, rounded corners=1.6pt, line width=0.5pt,
            fill=orange!25, draw=orange!65!black,
            minimum width=\subW cm, minimum height=0.40cm,
            align=center, font=\scriptsize, text=black]
          (an2) at (0, \hlY - 1.31) {Add\,$\oplus$\,RMSNorm};

      \draw[->, black, line width=0.5pt, >={Latex[length=1.6mm,width=1.4mm]}]
          (sl1.south) -- (an1.north);
      \draw[->, black, line width=0.5pt, >={Latex[length=1.6mm,width=1.4mm]}]
          (an1.south) -- (sl2.north);
      \draw[->, black, line width=0.5pt, >={Latex[length=1.6mm,width=1.4mm]}]
          (sl2.south) -- (an2.north);

      \def\BGap{0.20}
      \def\RGap{0.55}
      \pgfmathsetmacro{\BlueLx}{\blkWest - \BGap}
      \pgfmathsetmacro{\BlueRx}{\blkEast + \BGap}
      \pgfmathsetmacro{\BlueBy}{\blkBot - \BGap}
      \pgfmathsetmacro{\BlueTy}{\blkTop + \BGap}
      \pgfmathsetmacro{\RedLx}{\blkWest - \RGap}
      \pgfmathsetmacro{\RedRx}{\blkEast + \RGap}
      \pgfmathsetmacro{\RedBy}{\blkBot - \RGap}
      \pgfmathsetmacro{\RedTy}{\blkTop + \RGap}

      \draw[blue!72!black, line width=1.1pt, rounded corners=5pt,
            decoration={markings,
              mark=at position 0.85 with
                {\arrow[scale=1.6,line width=0.9pt]{Latex}}},
            postaction={decorate}]
          (\BlueRx, \BlueTy) -- (\BlueLx, \BlueTy)
          -- (\BlueLx, \BlueBy) -- (\BlueRx, \BlueBy) -- cycle;

      \draw[red!72!black, line width=1.1pt, rounded corners=6pt,
            decoration={markings,
              mark=at position 0.85 with
                {\arrow[scale=1.6,line width=0.9pt]{Latex}}},
            postaction={decorate}]
          (\RedRx, \RedTy) -- (\RedLx, \RedTy)
          -- (\RedLx, \RedBy) -- (\RedRx, \RedBy) -- cycle;

      \node[font=\scriptsize, fill=white, inner sep=2pt,
            align=center, text=black]
          at (-0.55, \BlueTy)
          {{\bfseries\color{blue!72!black}$\boldsymbol{\times L}$}\,inner ($f_L$)};
      \node[font=\scriptsize, fill=white, inner sep=2pt,
            align=center, text=black]
          at (+0.55, \RedTy)
          {{\bfseries\color{red!72!black}$\boldsymbol{\times H}$}};


      \node[proj] (s2w) at (0, \readY) {tokens $\mapsto w$};
      \draw[arrow] (hl.south) -- (s2w);
      \node[outc] (wR) at (0, \wY) {$w \in \mathbb{R}^{d}$};
      \draw[arrow] (s2w) -- (wR);
    \end{scope}

    \begin{scope}[on background layer]
      \node[draw, dotted, line width=0.7pt, gray!75,
            rounded corners=4pt, inner sep=10pt,
            fit=(LEFT)(RIGHT)] (BOTH_BOX) {};
    \end{scope}

    \draw[line width=0.8pt, black!55]
        let \p1 = ($(LEFT.east)!0.5!(RIGHT.west)$) in
        (\x1, {-3.90}) -- (\x1, {7.55});

  \end{tikzpicture}
  \caption{StyleGAN mapper~\citep{karras2019stylegan} (left) vs.\ our RTM (right): a shared block iterated $L$ times ({\color{blue!72!black}\textbf{blue}}) repeated for $H$ refinement steps ({\color{red!72!black}\textbf{red}}). The IMLE loss is applied only to the final style $w$; no supervision is applied between refinement steps.}
  \label{fig:mapper}
\end{figure}
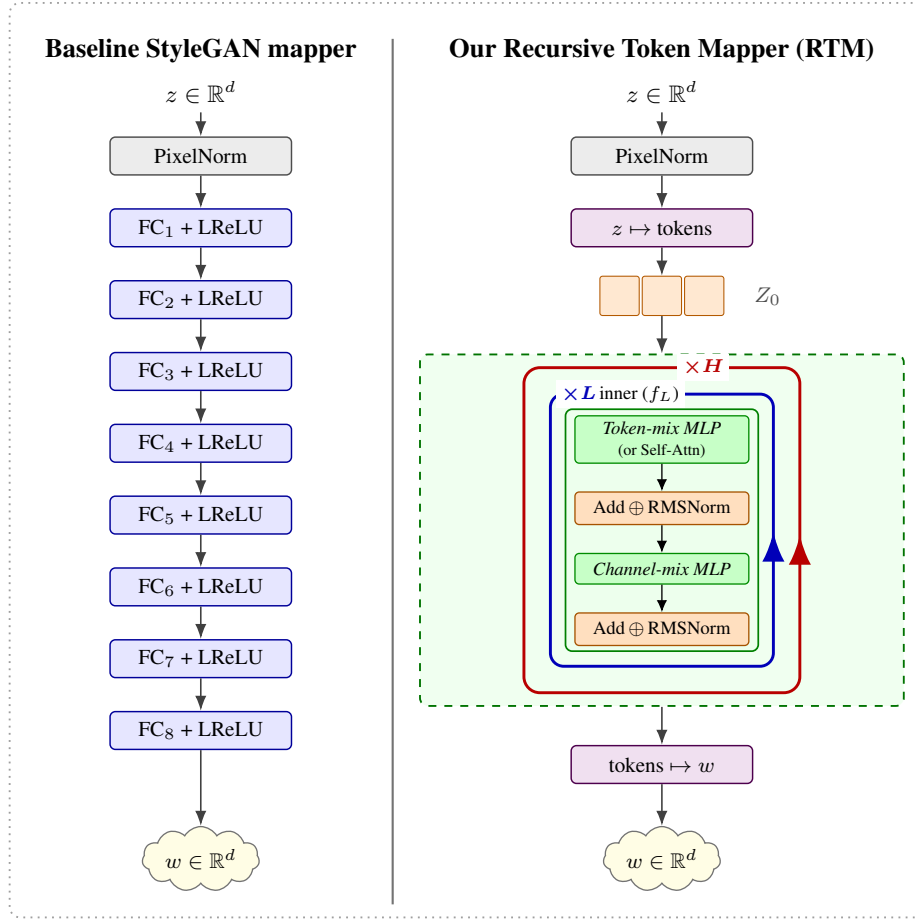

RTM produces the same style vector $w \in \mathbb{R}^d$ from the same noise vector $z \in \mathbb{R}^d$ as the MLP mapper it replaces, but it does so by repeatedly applying a single small block $f$ to a refined latent representation rather than passing $z$ through a deep stack of independent layers. The depth of computation is controlled by two integers: $L$ inner cycles and $H$ refinement steps. Because $f$ is reused at every iteration, increasing either count increases the effective depth of the mapper without adding parameters. Following~\citet{jolicoeurmartineau2025trm}, the two-level structure separates a fast-adapting inner state $Z_L$ (updated $L$ times per step) from a slower-accumulating outer state $Z_H$ (updated once per step), while re-injecting $Z_0$ at every inner cycle keeps the recursion anchored to the original noise.

\paragraph{Recursive H/L cycles.} After PixelNorm, $z$ is linearly lifted into a sequence of latent tokens $Z_0$, which seeds two states: an inner state $Z_L$ and an outer state $Z_H$, both initialized from fixed vectors. Within one step, $Z_L$ is updated $L$ times by $f$ conditioned on the current outer state and on $Z_0$; the refreshed $Z_L$ then drives a single update of $Z_H$ by $f$. This is repeated for $H$ steps in total, re-injecting $Z_0$ at every inner cycle so the recursion never loses contact with the original noise. The final $Z_H$ is read out by a linear layer to produce $w$. Algorithm~\ref{alg:rtm} in the appendix gives the full step-by-step pseudo-code, including the projection and readout layers.

\paragraph{Choice of block.} The shared block $f$ enables information exchange between latent tokens and is otherwise unconstrained. We use the MLP-Mixer-style token-mixing block of~\citet{tolstikhin2021mlpmixer} throughout this work: layer normalization, a SwiGLU MLP along the sequence axis to mix tokens, and a second SwiGLU MLP along the channel axis. This avoids the quadratic cost of attention while still permitting cross-token communication. To validate this design choice, we also ran the original TRM block of~\citet{jolicoeurmartineau2025trm}, which replaces token mixing with multi-head self-attention on the token grid, on the few-shot benchmarks (Section~\ref{sec:fewshot}). The MLP-Mixer variant was consistently better in FID and noticeably faster per step, and its advantage in both quality and wall-clock grows with sequence length and dataset size; we therefore use MLP-Mixer for all of the larger CIFAR-10, CelebA-HQ, AFHQ-v1, and StyleGAN runs.

\paragraph{Short-gradient optimization.} Recursing through $H$ steps would otherwise multiply the activation memory of the mapper by $H$, since the full computation graph of every step would need to be retained for backpropagation. To keep training memory tractable, only the final step is differentiated through; all earlier steps run without tracking gradients, so their intermediate activations are immediately discarded. This preserves the representational benefit of deep recursion while keeping the per-step training memory cost close to that of a single feed-forward block, analogous to truncated backpropagation through time.

\section{Experiments}

\paragraph{Experimental setup.} We evaluate RTM in two training regimes. In the IMLE regime (Sections~\ref{sec:cifar10}--\ref{sec:celebahq}), we integrate RTM as the mapping network of RS-IMLE and evaluate on unconditional CIFAR-10 at $32{\times}32$ and unconditional CelebA-HQ at $256{\times}256$. In the adversarial regime (Section~\ref{sec:stylegan}), we integrate RTM as the mapping network of StyleGAN2 and StyleGAN2-ADA and evaluate on unconditional CIFAR-10 at $32{\times}32$ and unconditional AFHQ-v1 at $512{\times}512$. Results on nine standard few-shot benchmarks are deferred to Appendix~\ref{sec:fewshot}. The decoder, regularizer, optimizer, and training schedule are identical between baseline and RTM runs in every comparison; only the mapping network changes. Per-dataset hyperparameters are in Appendix~\ref{app:training}.

\paragraph{Evaluation.} We use FID~\citep{heusel2017fid}, Inception Score~\citep{salimans2016is}, and the $k{=}3$ nearest-neighbour Precision and Recall of~\citet{kynkaanniemi2019pr} for a comprehensive evaluation; for the StyleGAN runs we additionally report Density and Coverage~\citep{naeem2020prdc}. FID measures the distance between Inception-v3 features of generated and real images; IS and Precision measure sample fidelity, Recall measures diversity, and Density and Coverage are kNN-based refinements of Precision and Recall. We compute all four metrics from the same Inception-v3 activations of 50{,}000, 30{,}000, and 5{,}000 generated samples for CIFAR-10, CelebA-HQ, and the few-shot benchmarks respectively, with the corresponding training set as the reference; for AFHQ-v1 we use the 14{,}630-image test split as the reference. Every reproduced row in our tables is evaluated by us from the corresponding method's official released checkpoint, so that all rows in a given table are comparable; rows marked with~$\dagger$ are taken from the original publication.

\subsection{Unconditional CIFAR-10}
\label{sec:cifar10}

\paragraph{Setup.} We evaluate on CIFAR-10~\citep{krizhevsky2009cifar} at $32{\times}32$. The matched comparison is RS-IMLE with the same decoder; Table~\ref{tab:cifar10} additionally covers representative GAN, diffusion / score-based / consistency, and flow-matching baselines (citations in the table).

\begin{table}[t]
  \caption{Unconditional CIFAR-10 ($32{\times}32$, 50{,}000 samples). Time / image is the wall-clock per generated image at batched inference on a single H100 (lower is faster); see App.~\ref{app:timing} for the protocol. \textbf{Bold}/\underline{underline}: best/second-best per metric column. $\dagger$~taken from the original publication.}
  \label{tab:cifar10}
  \centering
  \small
  \setlength{\tabcolsep}{3pt}
  \begin{tabular}{lccccc}
    \toprule
    Method & Precision $\uparrow$ & Recall $\uparrow$ & IS $\uparrow$ & FID $\downarrow$ & Time (ms) $\downarrow$ \\
    \midrule
    \multicolumn{6}{l}{\emph{GAN family (1-NFE, adversarial)}} \\
    StyleGAN2$^\dagger$~\citep{karras2020ada}                  & --                & --                & 9.21              & 8.32              & -- \\
    StyleGAN-XL~\citep{sauer2022styleganxl}                          & 0.674             & 0.467             & --                & \textbf{1.86}     & 3.6 \\
    \midrule
    \multicolumn{6}{l}{\emph{Diffusion / score-based / consistency family (multi-NFE unless noted)}} \\
    DDPM~\citep{ho2020ddpm}                                          & 0.619             & 0.567             & 8.62              & 11.18             & 13.3 \\
    NVAE+VAEBM$^\dagger$~\citep{xiao2021vaebm}                       & --                & --                & 8.43              & 12.19             & --   \\
    LSGM~\citep{vahdat2021lsgm}                                      & 0.703             & 0.596             & 10.07             & 2.79              & 570   \\
    ProxDM (hybrid)~\citep{fang2025proxdm}                           & 0.658             & 0.570             & 8.98              & 4.55              & 185   \\
    RGM (KLD-D, 4 steps)~\citep{choi2023rgm}                         & 0.653             & 0.526             & 9.28              & 4.85              & 1.4   \\
    EDM~\citep{karras2022edm}                                        & 0.675             & 0.618             & 9.91              & \underline{1.96}  & 14.6 \\
    Consistency Models (CD, 1-NFE)~\citep{song2023consistency}       & 0.688             & 0.560             & 9.72              & 3.57              & 1.8 \\
    Consistency Models (CT, 1-NFE)~\citep{song2023consistency}       & 0.702             & 0.419             & 8.61              & 8.79              & 1.5 \\
    \midrule
    \multicolumn{6}{l}{\emph{Flow-matching family (multi-NFE, no direct latent-to-image map)}} \\
    Flow Matching~\citep{lipman2023fm}                               & 0.651             & 0.589             & 9.28              & 3.72              & 23.8 \\
    OT-CFM~\citep{tong2023otcfm}                                     & 0.652             & 0.592             & 9.25              & 3.68              & 21.1 \\
    Mean Flows ($N{=}1$)~\citep{geng2025meanflow}                    & 0.687             & 0.586             & \underline{10.10} & 2.87              & 1.3 \\
    Mean Flows ($N{=}2$)~\citep{geng2025meanflow}                    & 0.704             & 0.582             & \textbf{10.13}    & 2.83              & 1.8 \\
    Inductive Moment Matching ($N{=}1$)~\citep{zhou2025imm}          & 0.659             & 0.593             & \underline{10.10} & 3.16              & 1.3 \\
    Inductive Moment Matching ($N{=}2$)~\citep{zhou2025imm}          & 0.674             & 0.615             & 10.08             & 2.01              & 1.7 \\
    \midrule
    \multicolumn{6}{l}{\emph{IMLE family (1-NFE, direct latent-to-image map)}} \\
    RS-IMLE Baseline                                      & \underline{0.853} & \underline{0.738} & 10.00             & 5.69              & 6.4 \\
    RS-IMLE + RTM $(H{=}16,L{=}1)$ \textbf{(Ours)}              & \textbf{0.896}    & \textbf{0.773}    & 10.08             & 3.97              & 6.7 \\
    \bottomrule
  \end{tabular}
\end{table}

\paragraph{Results.} Within the IMLE family, RTM lowers FID by 30\% over the matched RS-IMLE baseline while simultaneously improving Precision, Recall, and IS. Across families, RTM closes most of the remaining FID gap to the strongest diffusion and flow-matching baselines while retaining IMLE's one-step inference. The defining feature of our method is its position on the Precision/Recall axes: RTM achieves the highest Precision and the highest Recall of any method in Table~\ref{tab:cifar10}, despite being trained without an adversary or a diffusion solver. By contrast, flow-matching methods such as Mean Flows achieve a lower FID ($2.83$) but substantially lower Precision ($0.704$) and Recall ($0.582$), a pattern consistent with mode collapse: concentrating on high-density modes can improve FID while leaving large portions of the data distribution uncovered. This illustrates precisely why FID alone is an insufficient criterion for evaluating generative models.

\subsection{Unconditional CelebA-HQ at {\boldmath$256{\times}256$}}
\label{sec:celebahq}

\paragraph{Setup.} CelebA-HQ~\citep{karras2018progan} at $256{\times}256$, same RS-IMLE pipeline and decoder as described in Appendix~\ref{app:decoders}; only the mapping network changes. We compare against the matched RS-IMLE baseline and the strongest publicly-checkpointed unconditional CelebA-HQ generators (DDPM, DDGAN, RDM, StyleSwin); citations are in Table~\ref{tab:celebahq}. NVAE and VAEBM rows are taken verbatim from~\citet{vahdat2021lsgm}.

\begin{table}[t]
  \caption{Unconditional CelebA-HQ ($256{\times}256$, 30{,}000 samples). Time / image as in Table~\ref{tab:cifar10} (single H100, batched). \textbf{Bold}/\underline{underline}: best/second-best per metric column. Rows without $\dagger$ are our evaluations of each method's released checkpoint. RS-IMLE + RTM uses the $(H{=}16,L{=}2)$ configuration. }
  \label{tab:celebahq}
  \centering
  \small
  \begin{tabular}{lccccc}
    \toprule
    Method & Precision $\uparrow$ & Recall $\uparrow$ & IS $\uparrow$ & FID $\downarrow$ & Time (ms) $\downarrow$ \\
    \midrule
    \multicolumn{6}{l}{\emph{VAE / score-based family}} \\
    NVAE$^\dagger$~\citep{vahdat2020nvae}                            & --                & --                & --                & 29.76             & --    \\
    VAEBM$^\dagger$~\citep{xiao2021vaebm}                            & --                & --                & --                & 20.38             & --    \\
    LSGM~\citep{vahdat2021lsgm}                                      & 0.761             & 0.337             & 2.61              & 11.69             & 1{,}070    \\
    \midrule
    \multicolumn{6}{l}{\emph{Diffusion / score-based family}} \\
    DDPM~\citep{ho2020ddpm}                                          & 0.447             & 0.147             & 2.44              & 33.49             & 577   \\
    DDGAN~\citep{xiao2022ddgan}                                      & 0.683             & 0.205             & 2.71              & 15.83             & 8.8    \\
    ProxDM (hybrid)~\citep{fang2025proxdm}                           & 0.634             & 0.217             & 2.76              & 19.94             & 365   \\
    RDM~\citep{teng2024rdm}                                          & 0.709             & \underline{0.495}             & 3.29              & \textbf{5.77}     & 1{,}713 \\
    \midrule
    \multicolumn{6}{l}{\emph{GAN family}} \\
    StyleSwin~\citep{zhang2022styleswin}                             & 0.626             & 0.362             & \textbf{3.46}     & \underline{7.99}  & 39.6  \\
    \midrule
    \multicolumn{6}{l}{\emph{IMLE family (1-NFE, direct latent-to-image map)}} \\
    RS-IMLE Baseline                                     & \underline{0.924} & 0.491             & 3.18              & 15.43             & 18.1 \\
    RS-IMLE + RTM $(H{=}16, L{=}2)$ \textbf{(Ours)}          & \textbf{0.952}    & \textbf{0.592}    & \underline{3.36}  & 10.67             & 19.6 \\
    \bottomrule
  \end{tabular}
\end{table}

\paragraph{Results.} Within the IMLE family, RTM cuts FID by 27\% over the matched RS-IMLE baseline while simultaneously improving Precision, Recall, and IS. RDM is the strongest non-IMLE method by FID, but RTM achieves both the highest Precision and the highest Recall in the table while generating each sample in a single feed-forward pass rather than a multi-step diffusion solver. The remaining FID gap mirrors the trade-off observed on CIFAR-10 (Table~\ref{tab:cifar10}) and is the cost of insisting on a direct latent-to-image map.

\subsection{RTM as the mapper of a StyleGAN2 generator}
\label{sec:stylegan}

\paragraph{Setup.} To test whether the benefit of recursive mapping is specific to IMLE training, we plug RTM into the mapping network of StyleGAN2 and StyleGAN2-ADA~\citep{karras2020stylegan2,karras2020ada} and train them with the reproductions of those recipes from~\citet{kang2021ReACGAN}. We evaluate two settings: unconditional CIFAR-10 at $32{\times}32$ with StyleGAN2 and StyleGAN2-ADA, and unconditional AFHQ-v1 at $512{\times}512$ with StyleGAN2-ADA. The decoder, regularizer, augmentation pipeline, optimizer, and training schedule are identical between baseline and RTM runs; only the two-layer MLP mapper is replaced by an RTM with $(H,L){=}(16,1)$. All numbers in Table~\ref{tab:stylegan} are reported with the StudioGAN evaluation pipeline of~\citet{kang2021ReACGAN}, using Improved Precision/Recall~\citep{kynkaanniemi2019pr} and Density/Coverage~\citep{naeem2020prdc}; the FID and Precision/Recall implementations differ slightly in feature extractor, reference statistics, and image preprocessing from those used in Tables~\ref{tab:cifar10} and~\ref{tab:celebahq}.

\begin{table}[t]
  \caption{RTM as the StyleGAN2 / StyleGAN2-ADA mapper. Each (no RTM, +RTM) pair shares the same training pipeline; only the mapping network changes. Bold marks the better entry in each pair.}
  \label{tab:stylegan}
  \centering
  \small
  \setlength{\tabcolsep}{4pt}
  \begin{tabular}{llcccccc}
    \toprule
    Dataset & Method & FID $\downarrow$ & IS $\uparrow$ & Prec.\ $\uparrow$ & Rec.\ $\uparrow$ & Dens.\ $\uparrow$ & Cov.\ $\uparrow$ \\
    \midrule
    \multirow{4}{*}{\shortstack{CIFAR-10 \\($32{\times}32$)}}
      & StyleGAN2 (no RTM)                       & 3.88          & 10.20          & 0.734          & \textbf{0.664} & 0.987          & 0.894 \\
      & StyleGAN2 + RTM \textbf{(Ours)}           & \textbf{3.55} & \textbf{10.21} & \textbf{0.740} & 0.661          & \textbf{1.017} & \textbf{0.901} \\[2pt]
      & StyleGAN2-ADA (no RTM)                   & \textbf{2.31} & 10.46          & 0.744          & \textbf{0.685} & 1.050          & 0.932 \\
      & StyleGAN2-ADA + RTM \textbf{(Ours)}       & \textbf{2.31} & \textbf{10.50} & \textbf{0.754} & 0.669          & \textbf{1.063} & \textbf{0.933} \\
    \midrule
    \multirow{2}{*}{\shortstack{AFHQ-v1 \\($512{\times}512$)}}
      & StyleGAN2-ADA (no RTM)                   & 4.99          & \textbf{12.91} & 0.857          & 0.507          & \textbf{1.282} & \textbf{0.835} \\
      & StyleGAN2-ADA + RTM \textbf{(Ours)}       & \textbf{4.79} & 12.49          & \textbf{0.859} & \textbf{0.565} & 1.236          & 0.833 \\
    \bottomrule
  \end{tabular}
\end{table}

\paragraph{Results.} On CIFAR-10 with StyleGAN2, RTM lowers FID from the StudioGAN-reported $3.88$ to $3.55$ and improves IS, Precision, Density, and Coverage; Recall is essentially unchanged. With ADA augmentation, the RTM-mapped chain matches the StudioGAN baseline's FID of $2.31$ while improving IS, Precision, Density, and Coverage. On AFHQ-v1 with StyleGAN2-ADA, RTM lowers FID from $4.99$ to $4.79$ and increases Recall from $0.507$ to $0.565$; the baseline retains a small edge on IS, Density, and Coverage. Because the only difference between the paired rows is the mapping network, these gains show that the benefit of recursive mapping is not specific to IMLE training and transfers to an adversarial recipe.

\subsection{Analysis}

\paragraph{Varying the number of refinement steps at inference.} Because the IMLE loss is applied only to the final style $w$ and $H$ acts as a computational hyperparameter (Section~\ref{sec:generator}) rather than as a per-step training signal, the number of refinement steps used at inference can differ from the value used during training without any retraining or fine-tuning. After training the $(H{=}16, L{=}1)$ configuration on CIFAR-10 and CelebA-HQ, we re-evaluate the trained model at $H \in \{8, 16, 32, 64\}$, leaving every other component of the architecture and evaluation pipeline unchanged.

\begin{table}[t]
  \caption{Effect of varying the number of refinement steps $H$ at inference for the trained $(H{=}16, L{=}1)$ configuration. \textbf{Bold}: best per dataset.}
  \label{tab:halt_sweep}
  \centering
  \small
  \begin{tabular}{llcccc}
    \toprule
   Dataset & Inference-time $H$ & FID $\downarrow$ & IS $\uparrow$ & Precision $\uparrow$ & Recall $\uparrow$ \\
    \midrule
    \multirow{4}{*}{CIFAR-10}
      & $H{=}8$  (half of native)        & 4.04           & 10.07          & 0.896          & 0.774          \\
      & $H{=}16$ (native)                & 3.97           & 10.08          & 0.896          & 0.773          \\
      & $H{=}32$ (double)                & \textbf{3.94}  & \textbf{10.09} & \textbf{0.897} & \textbf{0.775} \\
      & $H{=}64$ (quadruple)             & \textbf{3.94}  & 10.08          & 0.895          & 0.774          \\
    \midrule
    \multirow{4}{*}{CelebA-HQ}
      & $H{=}8$  (half of native)        & 12.31          & 3.29           & \textbf{0.941} & 0.537          \\
      & $H{=}16$ (native)                & 12.22          & \textbf{3.30}  & 0.940          & \textbf{0.541} \\
      & $H{=}32$ (double)                & \textbf{12.19} & \textbf{3.30}  & 0.939          & \textbf{0.541} \\
      & $H{=}64$ (quadruple)             & 12.20          & \textbf{3.30}  & 0.939          & \textbf{0.541} \\
    \bottomrule
  \end{tabular}
\end{table}

On both datasets, FID improves modestly as $H$ increases from the trained value of $16$ up to $32$ and then plateaus at $H{=}64$, while Precision and Recall remain essentially constant across the entire sweep. A single trained model can therefore exchange a small amount of additional inference compute for a slight improvement in fidelity, or, conversely, halve its inference cost with only a marginal increase in FID, in either case without any retraining or fine-tuning.

\section{Conclusion}

Generative models are best evaluated with FID alongside Precision and
Recall, since FID alone conflates fidelity with mode coverage and
rewards sharp but low-diversity samples. Within that evaluation frame,
we introduced the Recursive Token Mapper (RTM), a drop-in replacement
for the single-pass MLP mapping network used by style-based generators. RTM is the first method leveraging TRM for continuous image generation (\citep{baekgenerative} first applied a generative TRM to binary black-and-white pictures). 
RTM gains effective depth through recursion rather than width, preserving IMLE's one-step inference. It consistently improves FID, Precision, and Recall across nine few-shot benchmarks, CIFAR-10, and CelebA-HQ, and also improves StyleGAN2 and StyleGAN2-ADA, demonstrating that the benefit is not specific to IMLE.

\paragraph{Limitations.} Our experiments are limited to the few-shot benchmarks, CIFAR-10, CelebA-HQ, and AFHQ-v1; we do not report on ImageNet, as IMLE's per-step cost (nearest-neighbour search) scales with dataset size, making ImageNet training infeasible within our compute budget. We leave large-scale IMLE to future work.

\paragraph{Broader Impact.} Better coverage of real data distributions benefits data augmentation and scientific image synthesis, and one-step inference improves accessibility. Improved generators could facilitate disinformation or deepfakes, though our incremental architectural change on standard benchmarks at modest resolutions limits direct misuse potential relative to large-scale systems already in deployment.

\paragraph{Future work.} HRM and TRM~\citep{wang2025hrm,jolicoeurmartineau2025trm} pair their recursive core with a learned halting head that allocates more compute to hard inputs and less to easy ones. A natural next step is a halting signal compatible with IMLE, so RTM can focus on hard latents (rare modes) without hand-picking $H$ at inference.

\bibliographystyle{plainnat}
\bibliography{references}


\newpage
\appendix
\section{Recursive Token Mapper: algorithmic description}
\label{app:rtm_algorithm}

Algorithm~\ref{alg:rtm} gives the full forward pass of RTM, including the
short-gradient optimization. A single IMLE loss is computed on the final
style $w$; no supervision is applied at intermediate steps.

\begin{algorithm}[h]
\caption{Recursive Token Mapper (RTM): Noise to Style}
\label{alg:rtm}
\begin{algorithmic}[1]
\REQUIRE Noise vector $z \in \mathbb{R}^d$, refinement steps $H$, inner cycles $L$
\ENSURE Style vector $w \in \mathbb{R}^d$
\STATE $Z_0 \leftarrow \text{Reshape}(W_{\text{proj}} \cdot z + b_{\text{proj}}) \in \mathbb{R}^{s \times d_h}$ \hfill // Project noise to tokens
\STATE $Z_H \leftarrow Z_H^{\text{init}}$, \; $Z_L \leftarrow Z_L^{\text{init}}$ \hfill // Initialize carry from fixed vectors
\FOR{$h = 1$ \TO $H$}
  \FOR{$\ell = 1$ \TO $L$}
    \STATE $Z_L \leftarrow f(Z_L, \; Z_H + Z_0)$ \hfill // L-level update with noise re-injection
  \ENDFOR
  \STATE $Z_H \leftarrow f(Z_H, \; Z_L)$ \hfill // H-level update
  \IF{$h < H$}
    \STATE $Z_H \leftarrow \text{detach}(Z_H), \; Z_L \leftarrow \text{detach}(Z_L)$ \hfill // Short-gradient: detach after non-final step
  \ENDIF
\ENDFOR
\STATE $w \leftarrow W_{\text{out}} \cdot \text{Flatten}(Z_H) + b_{\text{out}}$ \hfill // Readout to style vector
\RETURN $w$ \hfill // Loss is computed on $w$ only
\end{algorithmic}
\end{algorithm}

\section{Theoretical analysis}
\label{app:theory}

We make two short observations that justify replacing the StyleGAN MLP mapper with RTM. First, RTM keeps the IMLE coverage guarantee that motivates the loss. Second, RTM's compute budget is set by the inference-time schedule, not by its parameter count.

Throughout, write the generator as $T_\theta = G_\phi \circ M_\psi$, with mapper $M_\psi: \mathcal{Z} \to \mathcal{W}$ and decoder $G_\phi: \mathcal{W} \to \mathcal{X}$. Given training data $\{x_i\}_{i=1}^n$, a distance $d$, and a latent prior $p$, IMLE~\citep{li2018imle,vashist2024rsimle} draws $m$ candidate latents $z_1, \dots, z_m \sim p$ and minimizes
\begin{equation}
  \mathcal{L}_{\mathrm{IMLE}}(\theta) \;=\; \mathbb{E}_{z_{1:m}}\!\left[\;\sum_{i=1}^{n} \min_{j \in [m]} d\!\left(x_i,\, T_\theta(z_j)\right)\right].
  \label{eq:imle-app}
\end{equation}

\begin{lemma}[Coverage is preserved]
\label{lem:coverage}
Fix the decoder $G_\phi$, a training point $x$, and a tolerance $\varepsilon > 0$. Assume there is some style $w^\star$ with $d(x, G_\phi(w^\star)) \le \varepsilon/2$ and that $G_\phi$ is locally Lipschitz at $w^\star$. Then for any continuous mapper $M_\psi$ that maps some latent $z^\star$ with $p(z^\star) > 0$ close to $w^\star$, the probability that none of $m$ candidate latents lands within $\varepsilon$ of $x$ vanishes as $m \to \infty$.
\end{lemma}

\begin{proof}
Local Lipschitz-ness of $G_\phi$ at $w^\star$ gives a $\delta > 0$ such that $\|w - w^\star\| \le \delta$ implies $d(G_\phi(w), G_\phi(w^\star)) \le \varepsilon/2$. Continuity of $M_\psi$ at $z^\star$ then provides a neighbourhood $U$ of $z^\star$ that $M_\psi$ sends into $B(w^\star, \delta)$. Because $p(z^\star) > 0$, we have $\Pr[z \in U] = q > 0$, so the chance that none of the $m$ candidates lands in $U$ is at most $(1 - q)^m \to 0$. On the complementary event, the triangle inequality gives $d(x, T_\theta(z_j)) \le \varepsilon$ for that candidate.
\end{proof}

A standard MLP mapper and an RTM are both continuous compositions of differentiable layers, so the lemma applies to both: swapping in an RTM does not weaken the coverage guarantee. The same argument carries over to RS-IMLE~\citep{vashist2024rsimle}, which only changes the prior $p$ via rejection sampling and so preserves both positive density and continuity.

The second point is purely structural. The trainable parameters of an RTM are the projection, the readout, the shared block, and (when learnable) the carry initializations. None of these scales with the schedule $(H, L)$. From Algorithm~\ref{alg:rtm}, the shared block is evaluated $H \cdot (L + 1)$ times per sample. So compute can be turned up or down at inference time without changing the parameter count, which is what makes RTM parameter-efficient.

\section{Decoder architectures}
\label{app:decoders}

The mapping network is the only component we change; the convolutional
decoder is shared with each baseline. Figure~\ref{fig:decoders} shows the
per-dataset decoder pipelines used in our RS-IMLE experiments, and the table
beneath the diagrams gives their key hyper-parameters. Each pipeline starts
from a constant feature map at $1{\times}1$, progressively upsamples
through a stack of residual blocks, and emits an RGB image through a final
$1{\times}1$ convolution. The CIFAR-10 and CelebA-HQ runs use the same
ConvNeXt-style residual blocks~\citep{liu2022convnext}; the few-shot runs use the standard
$1{\times}1$/$3{\times}3$ residual blocks with GELU activations from adapted
RS-IMLE~\citep{vashist2024rsimle} codebase. The style code $w$ from the mapper modulates every block
via Adaptive Instance Normalization~\citep{huang2017adain}, and spatial
noise injection is applied at resolutions $\le 256$.

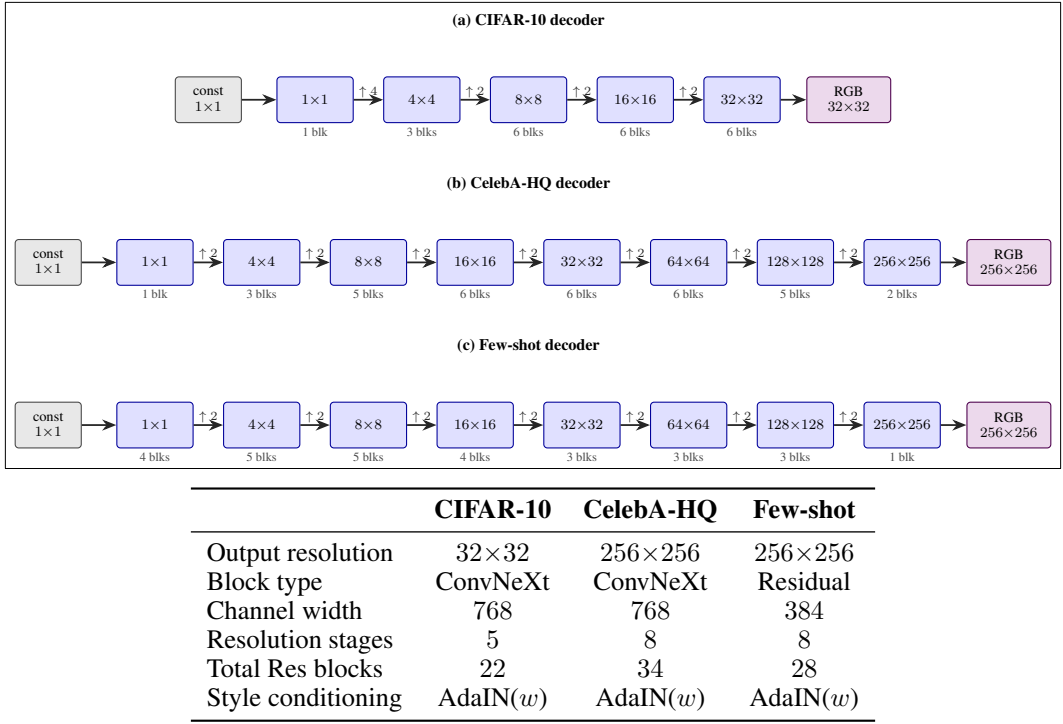
\begin{figure}[h]
  \centering
  \resizebox{\linewidth}{!}{%
  \fbox{\hspace{2pt}%
  \begin{tikzpicture}[
      font=\small,
      >={Stealth[length=2.8mm,width=2.2mm]},
      rowtitle/.style={font=\bfseries\small, anchor=center},
      stage/.style={draw, rounded corners=2pt, minimum width=1.50cm,
                    minimum height=0.90cm, align=center, line width=0.6pt,
                    font=\footnotesize, inner sep=2pt, fill=blue!12,
                    draw=blue!60!black},
      const/.style={draw, rounded corners=2pt, minimum width=1.30cm,
                    minimum height=0.90cm, align=center, line width=0.6pt,
                    font=\footnotesize, fill=gray!18, draw=gray!55!black},
      outbox/.style={draw, rounded corners=2pt, minimum width=1.65cm,
                     minimum height=0.90cm, align=center, line width=0.6pt,
                     font=\footnotesize, fill=violet!15, draw=violet!60!black},
      arr/.style={->, line width=1.0pt, black!85},
      uplab/.style={font=\scriptsize, inner sep=1pt},
      blklab/.style={font=\scriptsize, text=black!70, inner sep=1pt},
    ]
    \def\step{2.10}

    \node[rowtitle] at (4.5*\step, 1.55) {(a) CIFAR-10 decoder};
    \node[const]  (cC)   at (1.5*\step, 0)  {const\\[-1pt]$1{\times}1$};
    \node[stage]  (c1)   at (2.5*\step, 0)  {$1{\times}1$};
    \node[stage]  (c4)   at (3.5*\step, 0)  {$4{\times}4$};
    \node[stage]  (c8)   at (4.5*\step, 0)  {$8{\times}8$};
    \node[stage]  (c16)  at (5.5*\step, 0)  {$16{\times}16$};
    \node[stage]  (c32)  at (6.5*\step, 0)  {$32{\times}32$};
    \node[outbox] (cOut) at (7.5*\step, 0)  {RGB\\[-1pt]$32{\times}32$};
    \draw[arr] (cC) -- (c1);
    \draw[arr] (c1) -- node[above,uplab]{$\uparrow 4$} (c4);
    \draw[arr] (c4) -- node[above,uplab]{$\uparrow 2$} (c8);
    \draw[arr] (c8) -- node[above,uplab]{$\uparrow 2$} (c16);
    \draw[arr] (c16) -- node[above,uplab]{$\uparrow 2$} (c32);
    \draw[arr] (c32) -- (cOut);
    \node[blklab,below=1pt of c1]  {1 blk};
    \node[blklab,below=1pt of c4]  {3 blks};
    \node[blklab,below=1pt of c8]  {6 blks};
    \node[blklab,below=1pt of c16] {6 blks};
    \node[blklab,below=1pt of c32] {6 blks};

    \begin{scope}[yshift=-3.20cm]
      \node[rowtitle] at (4.5*\step, 1.55) {(b) CelebA-HQ decoder};
      \node[const]  (eC)   at (0*\step, 0)  {const\\[-1pt]$1{\times}1$};
      \node[stage]  (e1)   at (1*\step, 0)  {$1{\times}1$};
      \node[stage]  (e4)   at (2*\step, 0)  {$4{\times}4$};
      \node[stage]  (e8)   at (3*\step, 0)  {$8{\times}8$};
      \node[stage]  (e16)  at (4*\step, 0)  {$16{\times}16$};
      \node[stage]  (e32)  at (5*\step, 0)  {$32{\times}32$};
      \node[stage]  (e64)  at (6*\step, 0)  {$64{\times}64$};
      \node[stage]  (e128) at (7*\step, 0)  {$128{\times}128$};
      \node[stage]  (e256) at (8*\step, 0)  {$256{\times}256$};
      \node[outbox] (eOut) at (9*\step, 0)  {RGB\\[-1pt]$256{\times}256$};
      \draw[arr] (eC) -- (e1);
      \foreach \a/\b in {e1/e4, e4/e8, e8/e16, e16/e32, e32/e64, e64/e128, e128/e256}
        \draw[arr] (\a) -- node[above,uplab]{$\uparrow 2$} (\b);
      \draw[arr] (e256) -- (eOut);
      \node[blklab,below=1pt of e1]   {1 blk};
      \node[blklab,below=1pt of e4]   {3 blks};
      \node[blklab,below=1pt of e8]   {5 blks};
      \node[blklab,below=1pt of e16]  {6 blks};
      \node[blklab,below=1pt of e32]  {6 blks};
      \node[blklab,below=1pt of e64]  {6 blks};
      \node[blklab,below=1pt of e128] {5 blks};
      \node[blklab,below=1pt of e256] {2 blks};
    \end{scope}

    \begin{scope}[yshift=-6.40cm]
      \node[rowtitle] at (4.5*\step, 1.55) {(c) Few-shot decoder};
      \node[const]  (fC)   at (0*\step, 0)  {const\\[-1pt]$1{\times}1$};
      \node[stage]  (f1)   at (1*\step, 0)  {$1{\times}1$};
      \node[stage]  (f4)   at (2*\step, 0)  {$4{\times}4$};
      \node[stage]  (f8)   at (3*\step, 0)  {$8{\times}8$};
      \node[stage]  (f16)  at (4*\step, 0)  {$16{\times}16$};
      \node[stage]  (f32)  at (5*\step, 0)  {$32{\times}32$};
      \node[stage]  (f64)  at (6*\step, 0)  {$64{\times}64$};
      \node[stage]  (f128) at (7*\step, 0)  {$128{\times}128$};
      \node[stage]  (f256) at (8*\step, 0)  {$256{\times}256$};
      \node[outbox] (fOut) at (9*\step, 0)  {RGB\\[-1pt]$256{\times}256$};
      \draw[arr] (fC) -- (f1);
      \foreach \a/\b in {f1/f4, f4/f8, f8/f16, f16/f32, f32/f64, f64/f128, f128/f256}
        \draw[arr] (\a) -- node[above,uplab]{$\uparrow 2$} (\b);
      \draw[arr] (f256) -- (fOut);
      \node[blklab,below=1pt of f1]   {4 blks};
      \node[blklab,below=1pt of f4]   {5 blks};
      \node[blklab,below=1pt of f8]   {5 blks};
      \node[blklab,below=1pt of f16]  {4 blks};
      \node[blklab,below=1pt of f32]  {3 blks};
      \node[blklab,below=1pt of f64]  {3 blks};
      \node[blklab,below=1pt of f128] {3 blks};
      \node[blklab,below=1pt of f256] {1 blk};
    \end{scope}
  \end{tikzpicture}\hspace{2pt}}}

  \vspace{0.6em}

  \begin{tabular}{lccc}
    \toprule
                       & \textbf{CIFAR-10} & \textbf{CelebA-HQ} & \textbf{Few-shot} \\
    \midrule
    Output resolution  & $32{\times}32$    & $256{\times}256$   & $256{\times}256$ \\
    Block type         & ConvNeXt          & ConvNeXt           & Residual         \\
    Channel width      & $768$             & $768$              & $384$            \\
    Resolution stages  & $5$               & $8$                & $8$              \\
    Total Res blocks   & $22$              & $34$               & $28$             \\
    Style conditioning & AdaIN($w$)        & AdaIN($w$)         & AdaIN($w$)       \\
    \bottomrule
  \end{tabular}

  \caption{Decoder architectures used across our RS-IMLE experiments.}
  \label{fig:decoders}
\end{figure}

\section{Few-shot image generation}
\label{sec:fewshot}

We evaluate RTM on the nine standard few-shot image-generation benchmarks used by ~\citet{vashist2024rsimle}. Each benchmark contains only a few hundred training images to test RTM under limited data. 

\paragraph{Setup.} We use the nine standard few-shot benchmarks used by RS-IMLE ~\citet{vashist2024rsimle} (Obama, Grumpy Cat, Panda, FFHQ-100, Cat, Dog,
Anime, Skulls, Shells), each containing $64$--$389$ training images at
$256{\times}256$. All RS-IMLE runs share the same decoder, optimiser, and
rejection-sampling threshold; the only thing that changes between the matched RS-IMLE baseline and the RTM rows is the mapping network. RTM uses a single configuration $(H, L){=}(8,2)$ shared across all nine datasets, with no per-dataset tuning. We compare against FastGAN~\citep{liu2021fastgan},
AdaIMLE~\citep{aghabozorgi2023adaimle}, the published RS-IMLE
numbers~\citep{vashist2024rsimle}, and our own controlled reproduction of
RS-IMLE that uses an identical pipeline to the RTM runs. The rightmost column
of Table~\ref{tab:fid} ablates the original TRM attention block in place of
MLP-Mixer for the shared block~$f$.

\begin{table}[h]
  \caption{FID on nine few-shot benchmarks ($256{\times}256$, 5{,}000 samples). \textbf{Bold}: best per row.}
  \label{tab:fid}
  \centering
  \resizebox{\textwidth}{!}{%
  \begin{tabular}{lcccccc}
    \toprule
    Dataset & FastGAN & AdaIMLE & \makecell{RS-IMLE\\(Paper)} & \makecell{RS-IMLE\\(Reprod.)} & \makecell{Ours\\(MLP Mix)} & \makecell{Ours\\(Attention)} \\
    \midrule
    Obama      & 41.1  & 25.0  & 14.0  & 17.3  & \textbf{10.19} & 12.05 \\
    Grumpy Cat & 26.6  & 19.1  & 11.5  & 13.1  & \textbf{6.70}  & 11.53 \\
    Panda      & 10.0  & 7.6   & \textbf{3.5}   & 8.0   & 6.73  & 4.70  \\
    FFHQ-100   & 54.2  & 33.2  & \textbf{12.9}  & 19.8  & 13.23 & 13.68 \\
    Cat        & 35.1  & 24.9  & 15.9  & 17.5  & \textbf{8.62}  & 14.44 \\
    Dog        & 50.7  & 43.0  & 23.1  & 47.0  & \textbf{16.09} & 19.44 \\
    Anime      & 69.8  & 65.8  & 35.8  & 24.8  & 18.04 & \textbf{17.21} \\
    Skulls     & 109.6 & 81.9  & 51.1  & 42.4  & \textbf{21.36} & 43.01 \\
    Shells     & 120.9 & 108.5 & 55.4  & 37.15 & \textbf{24.48} & 26.37 \\
    \midrule
    Average    & 57.6  & 45.4  & 24.8  & 25.2  & \textbf{13.94} & 18.05 \\
    \bottomrule
  \end{tabular}%
  }
\end{table}

\paragraph{Results.} RTM with the MLP-Mixer block roughly halves the average FID of the matched RS-IMLE reproduction. Because the only thing that changes between the two rows is the mapping network, this improvement is attributable to the mapper's recursive structure rather than to capacity, optimizer, or training schedule. The published RS-IMLE numbers retain a small edge on a couple of individual datasets, but our controlled reproduction sits well above those numbers, suggesting that part of the published gap reflects per-dataset tuning that we did not attempt; against the matched reproduction, RTM wins on average and on the majority of benchmarks.

\paragraph{Choice of block.} The attention-based RTM ablation in the rightmost
column tracks the MLP-Mixer variant within a few FID points on most datasets
but is worse on average ($18.05$ vs.\ $13.94$) and is consistently slower per
training step because of the quadratic cost of self-attention on the sequence of tokens. The MLP-Mixer block is therefore used for all of the larger CIFAR-10,
CelebA-HQ, AFHQ-v1, and StyleGAN runs in the main paper. Qualitative samples for a selection of few-shot datasets are shown in Figures~\ref{fig:fewshot_shells}--\ref{fig:fewshot_anime}; per-dataset Precision and Recall numbers are reported in
Section~\ref{app:prec-recall}.

\section{Per-dataset Precision and Recall on few-shot benchmarks}
\label{app:prec-recall}

Table~\ref{tab:prec-recall} gives Precision and Recall per dataset for all few-shot benchmarks. Both RTM variants match or exceed the baselines in precision, and the attention variant delivers the second-highest average recall (0.88), only slightly below the strong RS-IMLE reproduction baseline (0.94), but with
substantially lower FID (Table~\ref{tab:fid}). The recall drop on datasets like Skulls and Shells is expected: with only a few dozen training images, the dataset itself is too small for recall to be a meaningful metric. On larger, more varied datasets such as Dog, RTM maintains strong Recall while improving Precision.

\begin{table}[h]
  \caption{Precision and Recall on the nine few-shot benchmarks ($256 \times 256$, 1{,}000 samples). \textbf{Bold}: best per metric.}
  \label{tab:prec-recall}
  \centering
  \resizebox{\textwidth}{!}{%
  \begin{tabular}{llcccccc}
    \toprule
    Dataset & & FastGAN & AdaIMLE & \makecell{RS-IMLE\\(Paper)} & \makecell{RS-IMLE\\(Reprod.)} & \makecell{Ours\\(MLP Mix)} & \makecell{Ours\\(Attention)} \\
    \midrule
    \multirow{2}{*}{Obama}
      & Prec. & 0.92 & \textbf{0.99} & 0.98 & \textbf{0.99} & \textbf{1.00} & \textbf{1.00} \\
      & Rec.  & 0.09 & 0.68 & 0.82 & \textbf{0.97} & 0.64 & 0.83 \\
    \midrule
    \multirow{2}{*}{Grumpy Cat}
      & Prec. & 0.91 & 0.97 & 0.93 & 0.99 & \textbf{1.00} & 0.99 \\
      & Rec.  & 0.13 & 0.72 & 0.95 & \textbf{0.97} & 0.65 & 0.94 \\
    \midrule
    \multirow{2}{*}{Panda}
      & Prec. & 0.96 & 0.98 & 0.99 & 0.99 & 0.99 & \textbf{1.00} \\
      & Rec.  & 0.16 & 0.63 & \textbf{0.97} & 0.84 & 0.93 & 0.90 \\
    \midrule
    \multirow{2}{*}{FFHQ-100}
      & Prec. & 0.91 & 0.99 & \textbf{1.00} & \textbf{1.00} & \textbf{1.00} & \textbf{1.00} \\
      & Rec.  & 0.13 & 0.77 & \textbf{0.99} & \textbf{0.99} & 0.65 & 0.89 \\
    \midrule
    \multirow{2}{*}{Cat}
      & Prec. & 0.97 & 0.98 & 0.96 & 0.97 & \textbf{1.00} & 0.98 \\
      & Rec.  & 0.08 & 0.86 & 0.98 & \textbf{0.99} & 0.80 & 0.98 \\
    \midrule
    \multirow{2}{*}{Dog}
      & Prec. & 0.96 & 0.97 & 0.98 & 0.97 & \textbf{0.99} & \textbf{0.99} \\
      & Rec.  & 0.19 & 0.61 & \textbf{0.94} & 0.76 & 0.92 & 0.89 \\
    \midrule
    \multirow{2}{*}{Anime}
      & Prec. & 0.86 & 0.92 & 0.95 & 0.97 & \textbf{0.99} & 0.98 \\
      & Rec.  & 0.08 & 0.59 & 0.91 & \textbf{1.00} & 0.75 & 0.96 \\
    \midrule
    \multirow{2}{*}{Skulls}
      & Prec. & 0.78 & 0.95 & \textbf{0.99} & 0.98 & \textbf{0.99} & \textbf{0.99} \\
      & Rec.  & 0.03 & 0.32 & 0.65 & \textbf{0.98} & 0.63 & 0.68 \\
    \midrule
    \multirow{2}{*}{Shells}
      & Prec. & 0.92 & 0.97 & 0.98 & 0.99 & 0.99 & \textbf{1.00} \\
      & Rec.  & 0.03 & 0.62 & 0.59 & \textbf{0.97} & 0.66 & 0.81 \\
    \midrule
    \multirow{2}{*}{\textbf{Average}}
      & Prec. & 0.91 & 0.97 & 0.97 & 0.98 & \textbf{0.99} & \textbf{0.99} \\
      & Rec.  & 0.10 & 0.64 & 0.87 & \textbf{0.94} & 0.74 & 0.88 \\
    \bottomrule
  \end{tabular}%
  }
\end{table}

\section{Implementation details and training stability}
\label{app:training}

\paragraph{Optimization.} All RS-IMLE runs use Adam with $\beta_1 = 0.5$,
$\beta_2 = 0.999$. All
dataset-specific hyperparameters, including learning rate, batch size,
training schedule, and RTM configuration $(H, L)$, are provided in the
configuration files released with the code.

\paragraph{Short-gradient Optimization.} As described in
Section~\ref{sec:generator}, only the final refinement step is differentiated
through; all preceding steps are detached. This is the dominant memory
saving in RTM, and is what allows us to train deep configurations
($H{=}8, L{=}2$ etc.) on a single GPU.

\paragraph{Compute.} All experiments were run on NVIDIA H100 GPUs.
RS-IMLE runs (baseline and RTM) on CIFAR-10 take approximately two weeks 
on 4 H100 GPUs; on CelebA-HQ at $256{\times}256$ they take approximately 
three weeks on 4 H100 GPUs. Few-shot runs (baseline and RTM) take 
approximately 16 hours on a single H100 GPU for the smaller benchmarks 
(Shells, Skulls) and up to 2 days for the larger ones (Dog). StyleGAN2 
runs (baseline and RTM) on CIFAR-10 take approximately 24 hours on a 
single H100 GPU. StyleGAN2-ADA runs (baseline and RTM) on CIFAR-10 take 
approximately 4 days on a single H100 GPU. StyleGAN2-ADA runs (baseline 
and RTM) on AFHQ-v1 at $512{\times}512$ take approximately 4 days on 4 
H100 GPUs.

\paragraph{Dataset Licenses.}
CIFAR-10 is freely available for research and commercial use; we cite 
the original technical report~\citep{krizhevsky2009cifar}. 
CelebA-HQ~\citep{karras2018progan} is restricted to non-commercial 
research and educational use. AFHQ-v1~\citep{choi2020starganv2} is 
released under the Creative Commons Attribution-NonCommercial 4.0 
International (CC BY-NC 4.0) license. The StyleGAN2 and StyleGAN2-ADA 
codebases~\citep{karras2020stylegan2,karras2020ada} are released under 
the NVIDIA Source Code License. We use all datasets and codebases 
strictly for non-commercial research purposes

\section{StyleGAN Evaluation Protocol}
\label{app:studiogan-fid}

Table~\ref{tab:stylegan} uses the StudioGAN evaluation pipeline of~\citet{kang2021ReACGAN}, which implements Improved Precision/Recall~\citep{kynkaanniemi2019pr} and Density/Coverage~\citep{naeem2020prdc}. The CIFAR-10 ``StyleGAN2 (no RTM)'' row reports the best-FID checkpoint of the StudioGAN baseline at $170{,}000$ steps. The FID and Precision/Recall implementations used in this pipeline differ slightly in feature extractor, reference statistics, and image preprocessing from those used in Tables~\ref{tab:cifar10} and~\ref{tab:celebahq}.

\section{Inference Latency}
\label{app:timing}

Inference latency is measured as the amortized wall-clock time per generated
image on a single NVIDIA H100 GPU. We synchronize the device, time one
complete forward pass of the full generator (EMA weights, no gradient) over
a batch of noise vectors, and divide the elapsed time by the batch size.
We use batch size $64$ for CIFAR-10 and batch size $16$ for CelebA-HQ at
$256{\times}256$, reflecting memory constraints at the higher resolution.
Each reported value is the median batch time over 200 forward passes,
preceded by 20 warm-up passes to stabilize GPU cache state.

\section{Depth-only Ablation: a deeper non-recursive MLP mapper}
\label{app:depth-ablation}

Our main StyleGAN2 results in Table~\ref{tab:stylegan} replace a $2$-layer MLP mapping network with an RTM that applies the shared block $H \cdot (L+1) = 16 \cdot 2 = 32$ times per forward pass (Appendix~\ref{app:theory}). To check whether the gain comes from depth (more sequential transformations of $z$) rather than from recursion (the same parameters reused across cycles), we train two additional StyleGAN2 variants that swap the $2$-layer MLP for a non-recursive $16$- or $32$-layer MLP and are otherwise identical to the StudioGAN baseline. The $32$-layer MLP is the true depth-matched baseline: it applies the same number of sequential transformations as RTM. 

\begin{table}[h]
  \caption{Depth-only ablation on CIFAR-10 with StyleGAN2: only the mapping network changes. The RTM config $(H, L){=}(16,1)$ applies the shared block $H{\cdot}(L{+}1)=32$ times per forward pass, so the 32-layer MLP is the depth-matched non-recursive baseline. \textbf{Bold}: best per metric.}
  \label{tab:depth-ablation}
  \centering
  \small
  \begin{tabular}{lccccc}
    \toprule
    Mapping network & Mapper params & FID $\downarrow$ & IS $\uparrow$ & Prec.\ $\uparrow$ & Rec.\ $\uparrow$ \\
    \midrule
    2-layer MLP (StudioGAN baseline)            & 0.53M  & 3.88              & 10.20             & 0.734             & \textbf{0.664}   \\
    16-layer MLP (half-depth, non-recursive)    & 4.2M   & 3.73              & 10.18             & 0.751             & 0.616            \\
    32-layer MLP (depth-matched, non-recursive) & 8.4M   & 4.32              & \textbf{10.51}    & \textbf{0.769}    & 0.550            \\
    RTM, $(H,L){=}(16,1)$ \textbf{(Ours)}   & 0.66M  & \textbf{3.55}     & 10.21             & 0.740             & 0.661            \\
    \bottomrule
  \end{tabular}
\end{table}

\paragraph{Setup.} Each MLP variant takes the same StyleGAN2 backbone as the StudioGAN baseline and only deepens the mapping network from $2$ to $16$ or $32$ layers; everything else (decoder, discriminator, optimizer, training schedule) is unchanged. All rows are trained for the same $200{,}000$ steps on the same CIFAR-10 split, and we report the best-FID checkpoint of each run. RTM uses the MLP-Mixer block with $(H, L) = (16, 1)$; per the formula in Appendix~\ref{app:theory}, this gives $H \cdot (L+1) = 32$ shared-block evaluations per forward pass, matching the sequential depth of the $32$-layer MLP.

\paragraph{Observation.} Increasing non-recursive depth from $2$ to $16$ layers improves FID ($3.73$ vs.\ $3.88$), confirming that the mapping network is sensitive to sequential depth. Further deepening to $32$ layers reverses this trend ($4.32$ FID), suggesting that very deep non-recursive MLPs overfit the mapping task without the regularising effect of parameter sharing. The $32$-layer MLP achieves the highest Precision ($0.769$) and IS ($10.51$), but at a severe cost to Recall ($0.550$ vs.\ $0.664$ for the $2$-layer baseline), consistent with a high-capacity non-recursive mapper that concentrates probability mass on the high-density modes of the data distribution. RTM achieves the best FID ($3.55$) with $13\times$ fewer parameters than the depth-matched MLP ($0.66$M vs.\ $8.4$M), and maintains Recall close to the $2$-layer baseline ($0.661$ vs.\ $0.664$). The recursive structure thus provides a more effective inductive bias than raw depth: parameter sharing across cycles improves distributional fidelity without sacrificing coverage.

\section{Qualitative few-shot samples}
\label{app:fewshot_samples}

Figures~\ref{fig:fewshot_shells}--\ref{fig:fewshot_anime} show random
samples from our MLP token-mixing RTM on four few-shot benchmarks.
Figures~\ref{fig:fewshot_slerp_p1} and~\ref{fig:fewshot_slerp_p2} show
SLERP interpolations in latent space for the same four benchmarks.

\begin{figure}[!t]
  \centering
  \includegraphics[width=\linewidth]{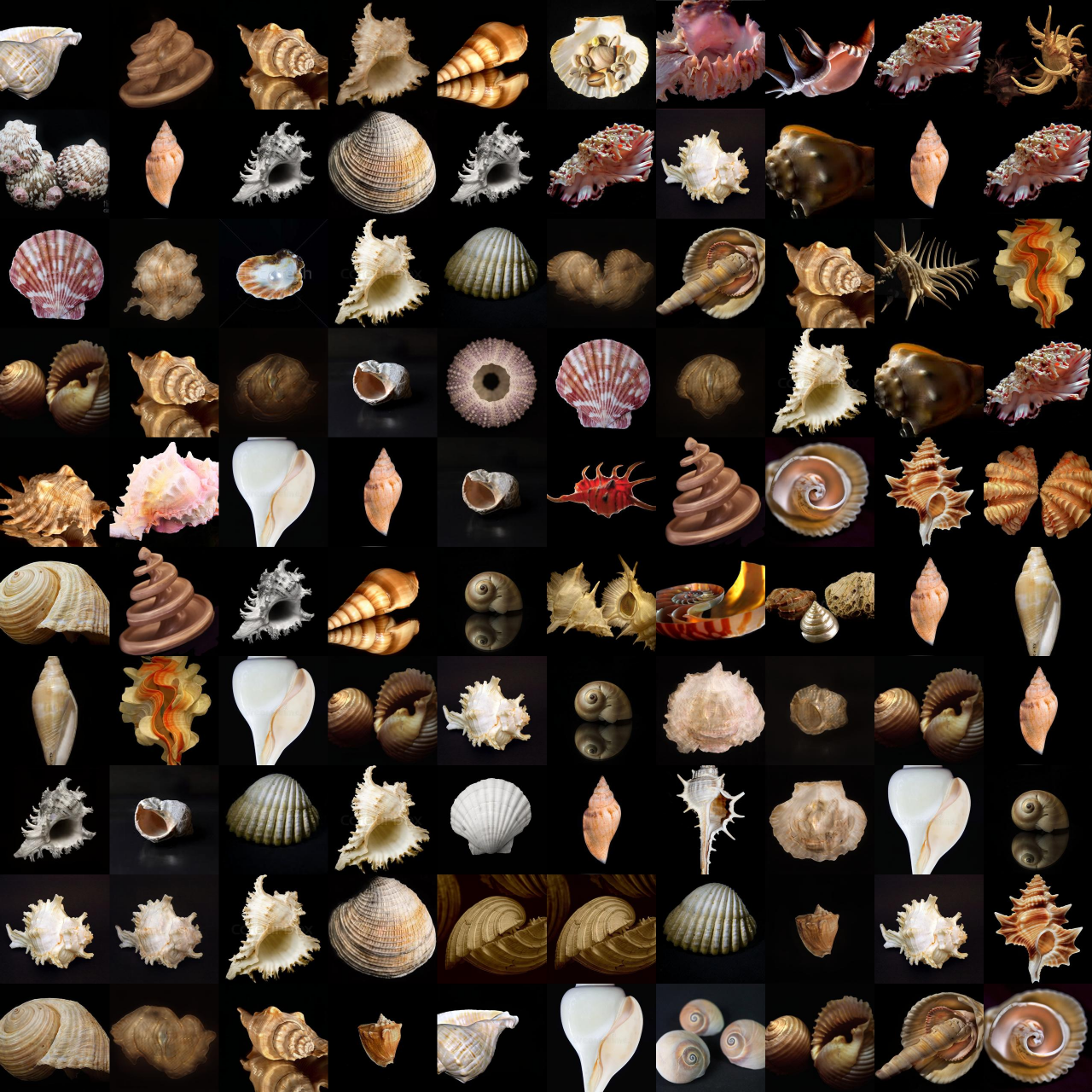}
  \caption{Random samples from RS-IMLE + RTM on Shells.}
  \label{fig:fewshot_shells}
\end{figure}
\clearpage

\begin{figure}[!t]
  \centering
  \includegraphics[width=\linewidth]{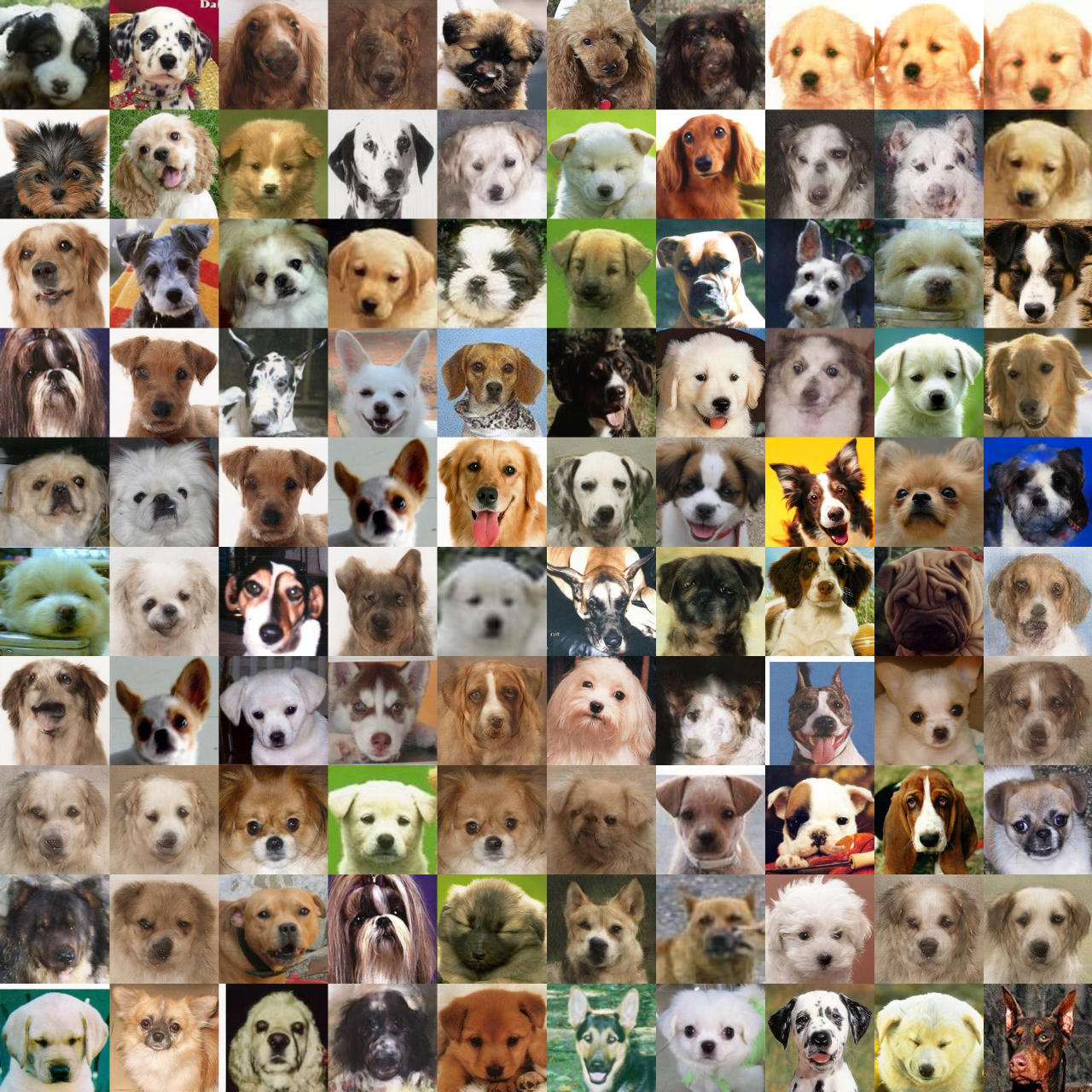}
  \caption{Random samples from RS-IMLE + RTM on Dog.}
  \label{fig:fewshot_dog}
\end{figure}
\clearpage

\begin{figure}[!t]
  \centering
  \includegraphics[width=\linewidth]{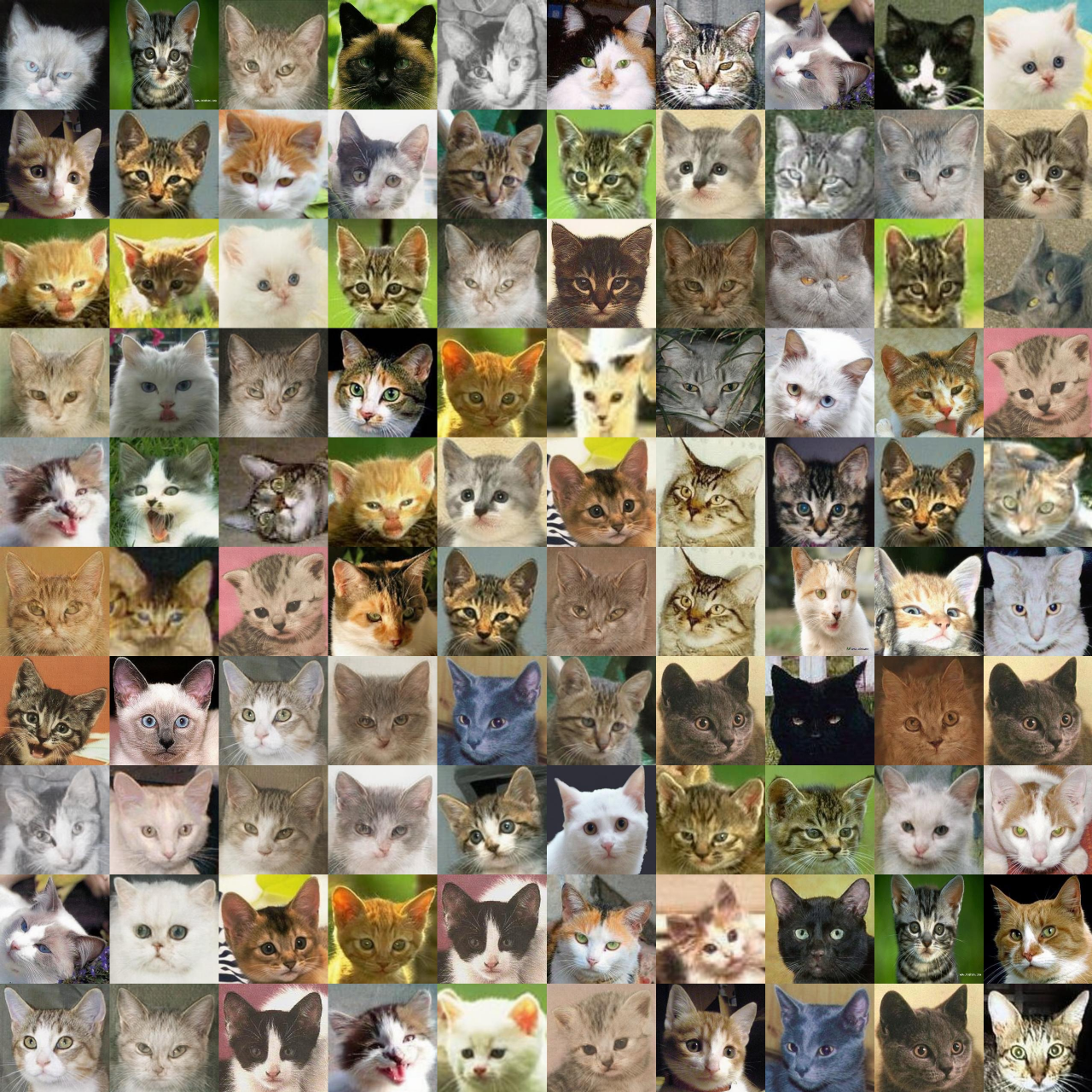}
  \caption{Random samples from RS-IMLE + RTM on Cat.}
  \label{fig:fewshot_cat}
\end{figure}
\clearpage

\begin{figure}[!t]
  \centering
  \includegraphics[width=\linewidth]{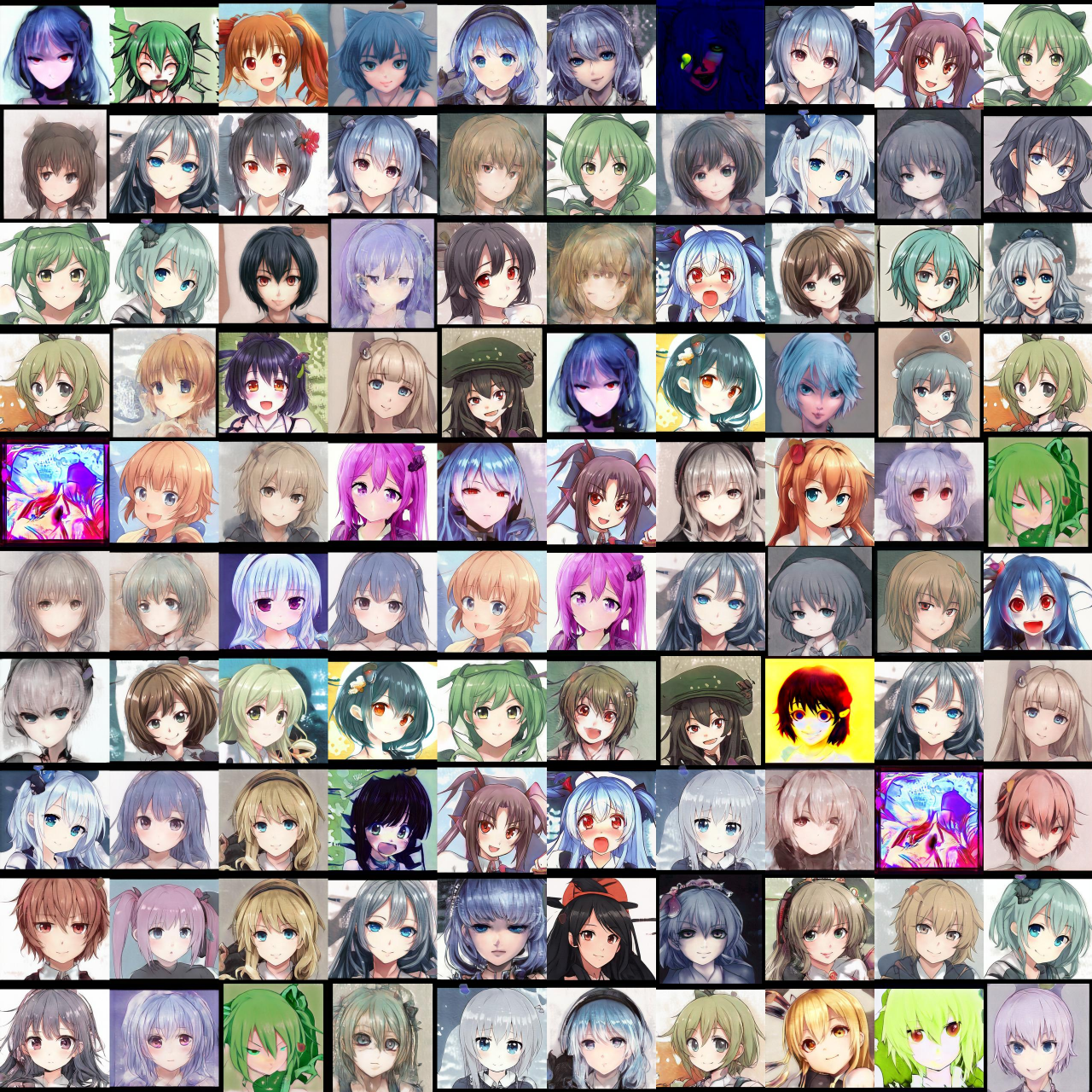}
  \caption{Random samples from RS-IMLE + RTM on Anime.}
  \label{fig:fewshot_anime}
\end{figure}
\clearpage

\begin{figure}[!t]
  \centering
  \begin{subfigure}[b]{\linewidth}
    \includegraphics[width=\linewidth]{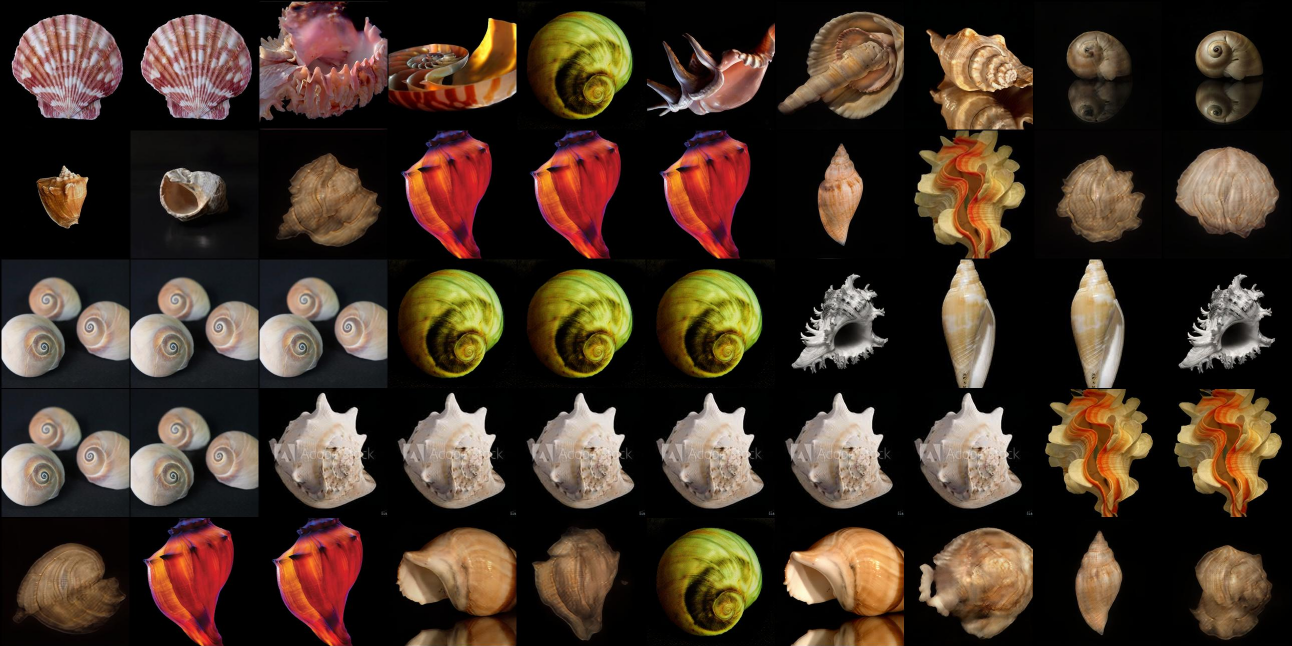}
    \caption{Shells}
  \end{subfigure}
  \\[0.8em]
  \begin{subfigure}[b]{\linewidth}
    \includegraphics[width=\linewidth]{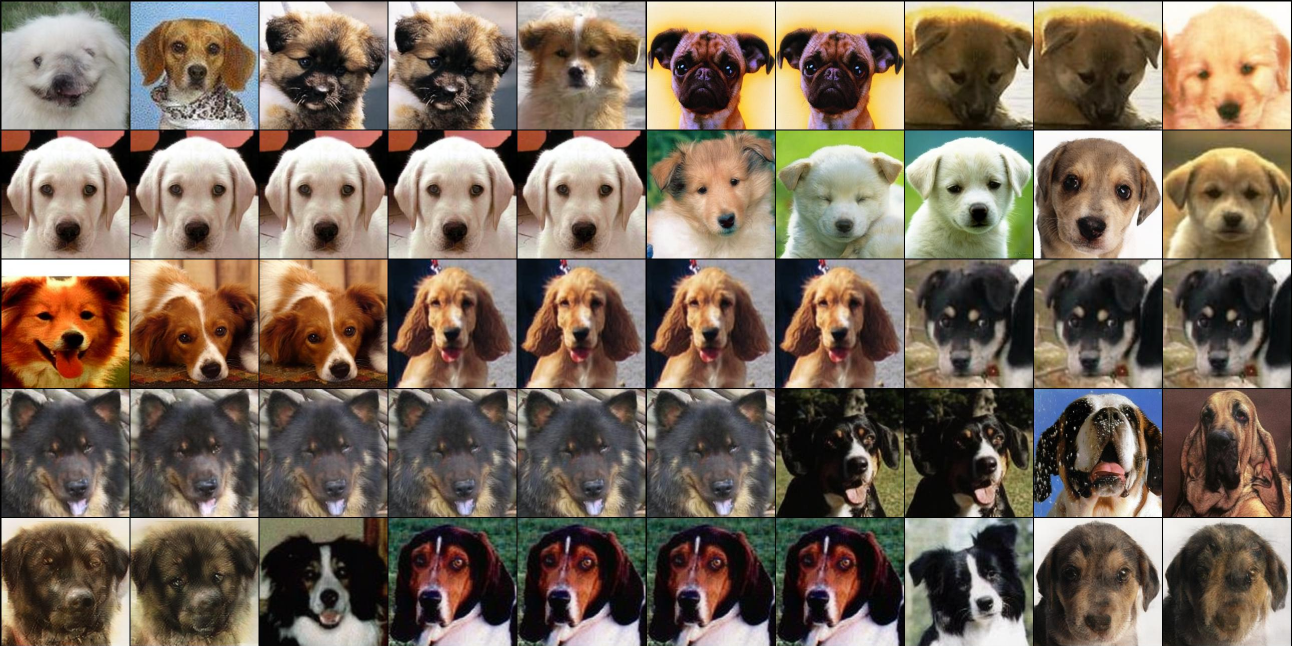}
    \caption{Dog}
  \end{subfigure}
  \caption{SLERP interpolations in latent space from RS-IMLE + RTM on Shells and Dog.}
  \label{fig:fewshot_slerp_p1}
\end{figure}
\clearpage

\begin{figure}[!t]
  \centering
  \begin{subfigure}[b]{\linewidth}
    \includegraphics[width=\linewidth]{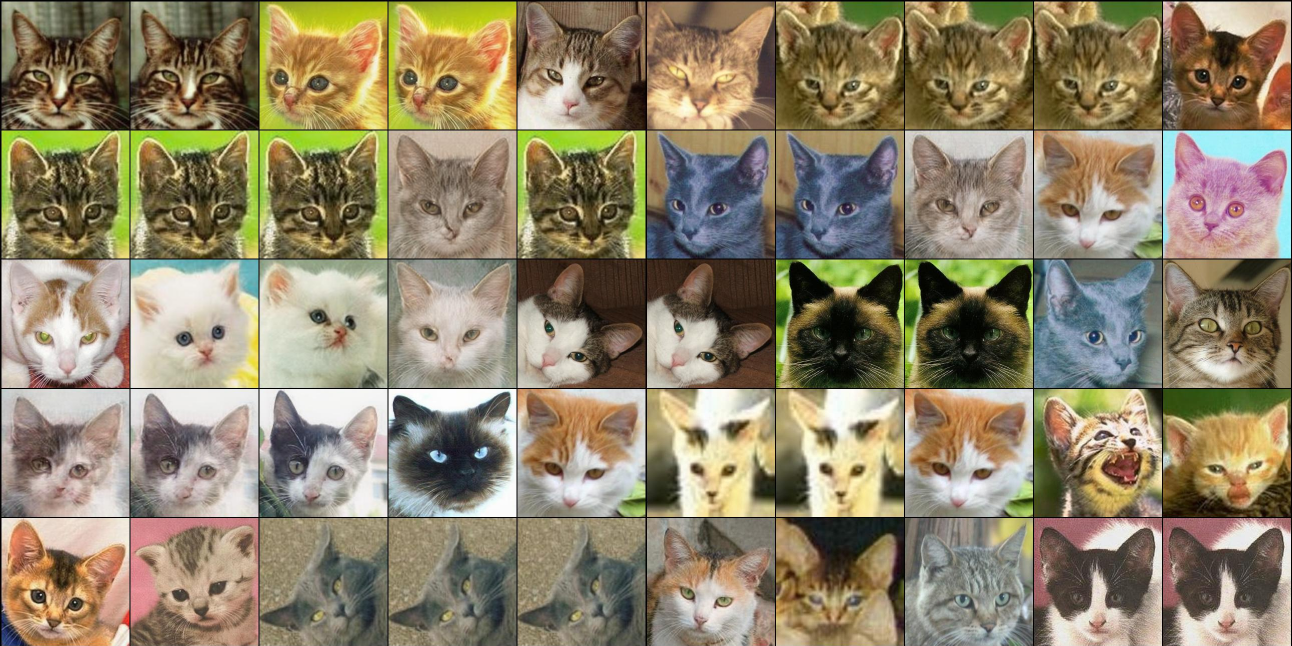}
    \caption{Cat}
  \end{subfigure}
  \\[0.8em]
  \begin{subfigure}[b]{\linewidth}
    \includegraphics[width=\linewidth]{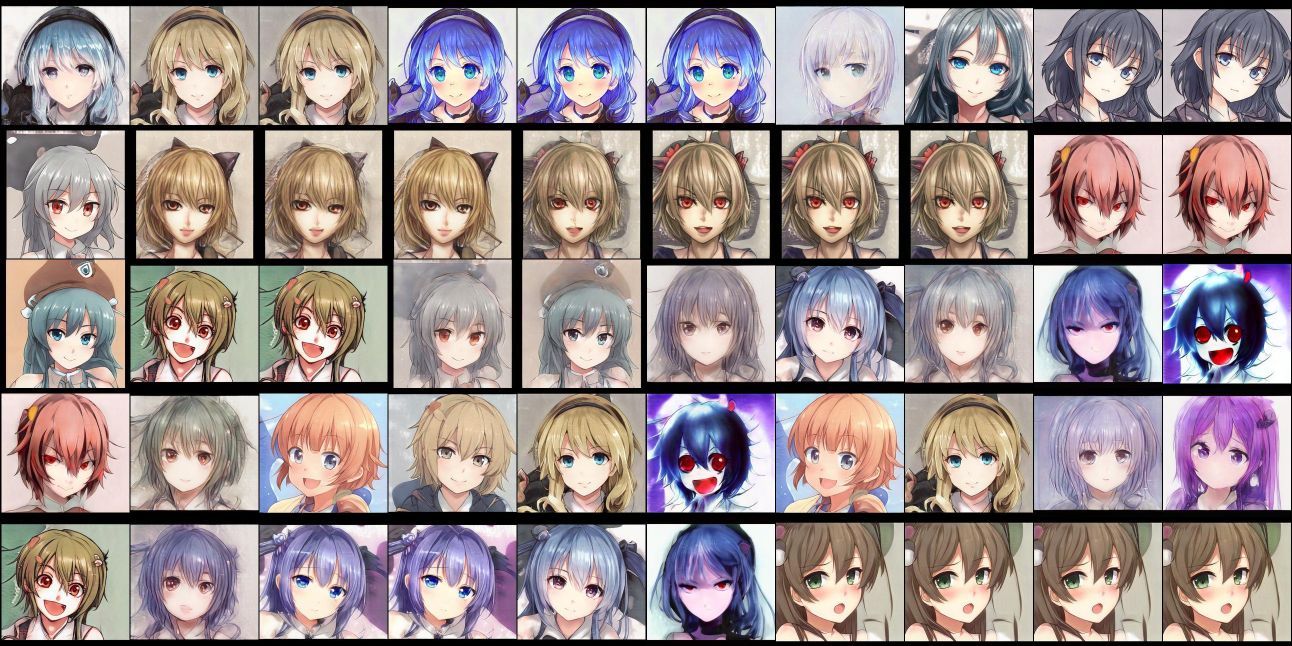}
    \caption{Anime}
  \end{subfigure}
  \caption{SLERP interpolations in latent space from RS-IMLE + RTM on Cat and Anime.}
  \label{fig:fewshot_slerp_p2}
\end{figure}
\clearpage

\section{Qualitative CIFAR-10 samples}
\label{app:cifar_samples}

Figure~\ref{fig:diversity_cifar10_rtm} shows $\sim$4{,}000 random samples from our $(H{=}16, L{=}1)$ RTM on CIFAR-10. At this density, every CIFAR-10 class is represented hundreds of times and within-class appearance varies in colour, pose, scale, and background, consistent with the high Recall reported in Table~\ref{tab:cifar10}.

\begin{figure}[!t]
  \centering
  \includegraphics[width=\linewidth,height=0.92\textheight,keepaspectratio]{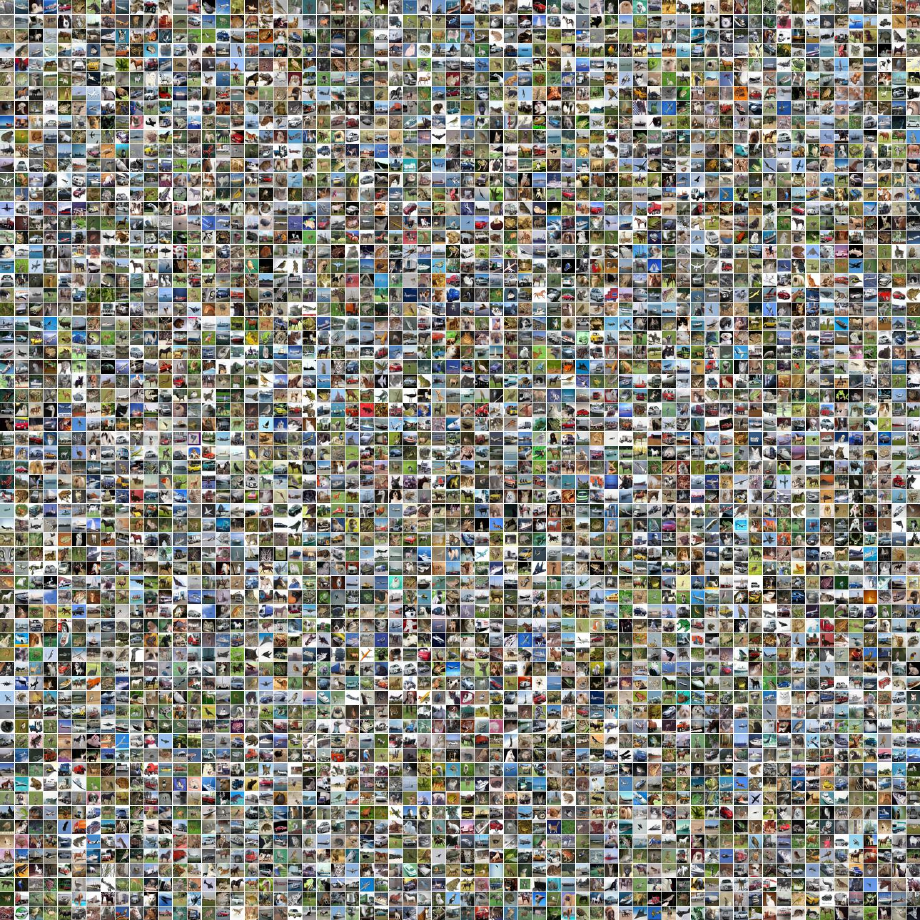}
  \caption{$\sim$4{,}000 unconditional CIFAR-10 samples from our RS-IMLE + RTM, $(H{=}16, L{=}1)$. Best viewed zoomed in.}
  \label{fig:diversity_cifar10_rtm}
\end{figure}
\clearpage

\section{Baseline-vs-RTM qualitative comparisons}
\label{app:baseline_vs_rtm}

Every block in Figures~\ref{fig:trm_nn_celeba} and~\ref{fig:trm_nn_cifar}
is anchored on a single query image taken from the real dataset and paired
with the nearest neighbours from each model's generated pool in
Inception feature space. The top row shows neighbours from the RS-IMLE
baseline, and the bottom row shows neighbours from RS-IMLE + RTM. The model
whose neighbours more faithfully reproduce the query is the better matcher.

\begin{figure}[!t]
  \centering
  \includegraphics[width=\linewidth]{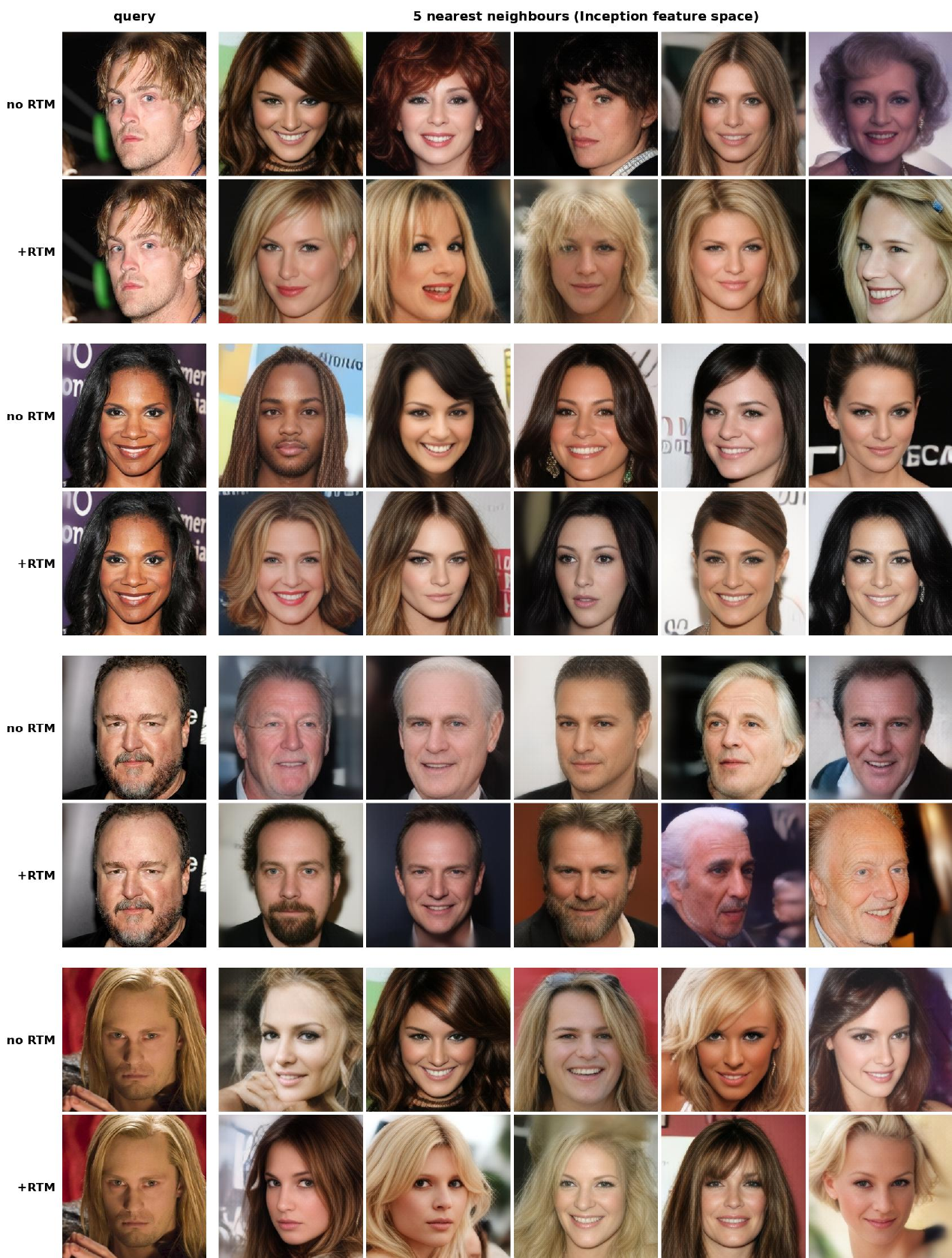}
  \caption{RS-IMLE+RTM neighbours more faithfully match the query across gender, skin tone, age, and hair attributes that the baseline often fails to preserve.}
  \label{fig:trm_nn_celeba}
\end{figure}
\clearpage

\begin{figure}[!t]
  \centering
  \includegraphics[width=\linewidth,height=0.92\textheight,keepaspectratio]{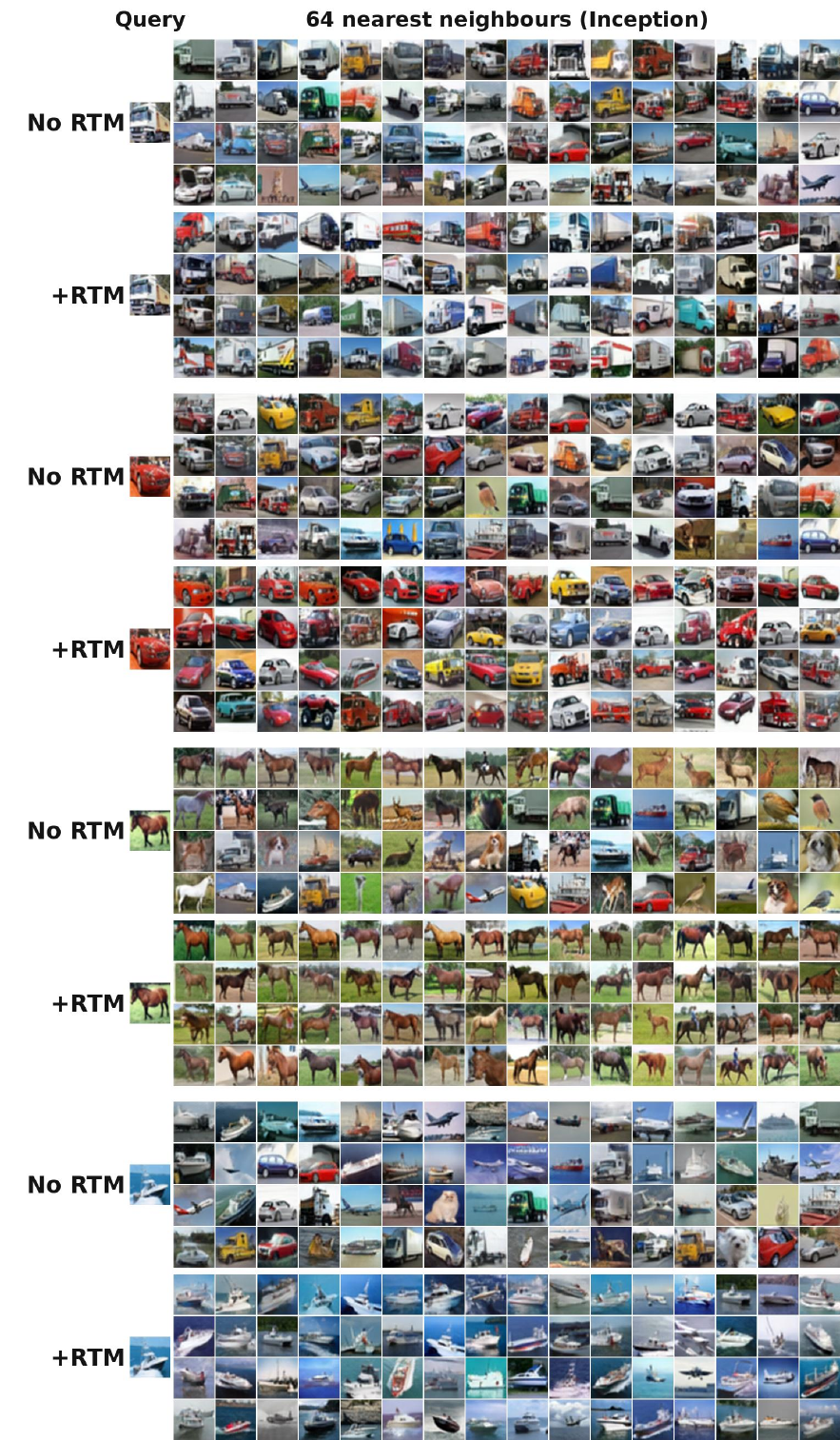}
  \caption{RS-IMLE+RTM neighbours better cluster around the query's class and visual characteristics, reflecting improved mode coverage over the baseline.}
  \label{fig:trm_nn_cifar}
\end{figure}
\clearpage

\section{CelebA-HQ baseline vs.\ RTM}
\label{app:celeba_comparison}

Figure~\ref{fig:celeba_baseline_vs_rtm} compares RS-IMLE and RS-IMLE + RTM
on CelebA-HQ at $256{\times}256$. The top panel shows matched sample pairs. The bottom panel shows a broader set of generations from each model, illustrating the gain in both sample quality and diversity consistent with
the Precision and Recall improvements in Table~\ref{tab:celebahq}.

\begin{figure}[!t]
  \centering
  \includegraphics[width=\linewidth]{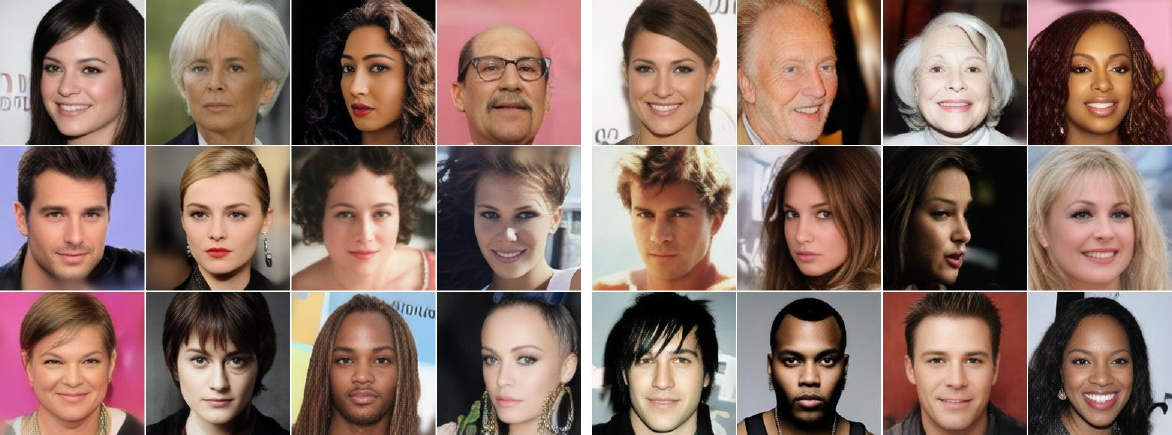}
  \\[0.8em]
  \includegraphics[width=\linewidth]{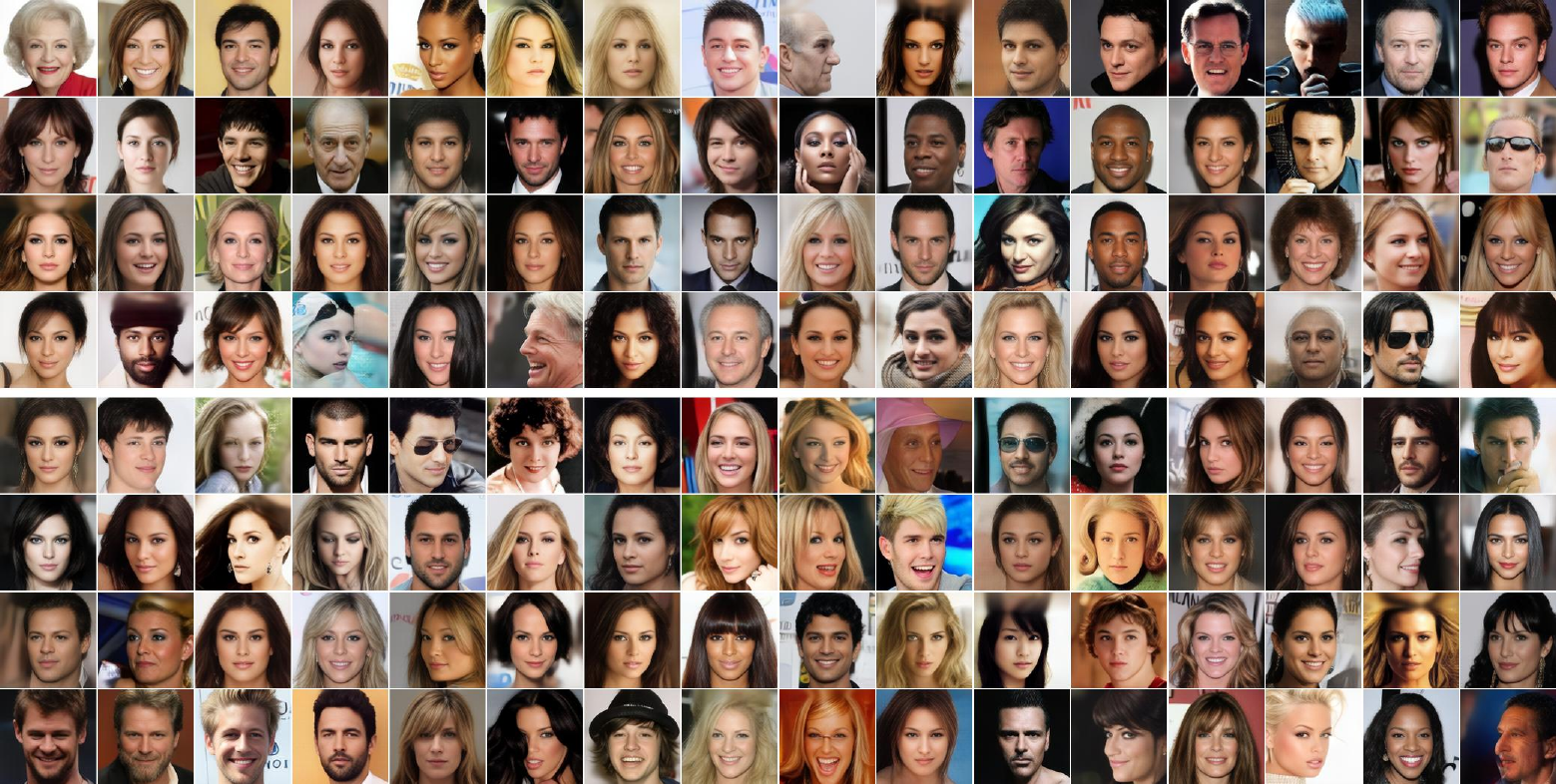}
  \caption{CelebA-HQ $256{\times}256$ comparison between RS-IMLE baseline and RS-IMLE + RTM. Top: left is without RTM, right is with RTM. Bottom: top rows are without RTM, bottom rows are with RTM. RTM generates sharper images with greater variety in age, skin tone, and expression, consistent with the improved Precision and Recall in Table~\ref{tab:celebahq}.}
  \label{fig:celeba_baseline_vs_rtm}
\end{figure}
\clearpage

\section{Additional samples}
\label{app:showcase}

Figures~\ref{fig:showcase_celeba_p1}--\ref{fig:showcase_celeba_p2} show unconditional CelebA-HQ $256{\times}256$ samples
from RS-IMLE + RTM.
Figures~\ref{fig:showcase_afhq_p1} and \ref{fig:showcase_afhq_p2} show
additional baseline-vs-RTM AFHQ-v1 comparisons for the StyleGAN2-ADA setup.

\begin{figure}[!t]
  \centering
  \includegraphics[width=\linewidth,height=0.92\textheight,keepaspectratio]{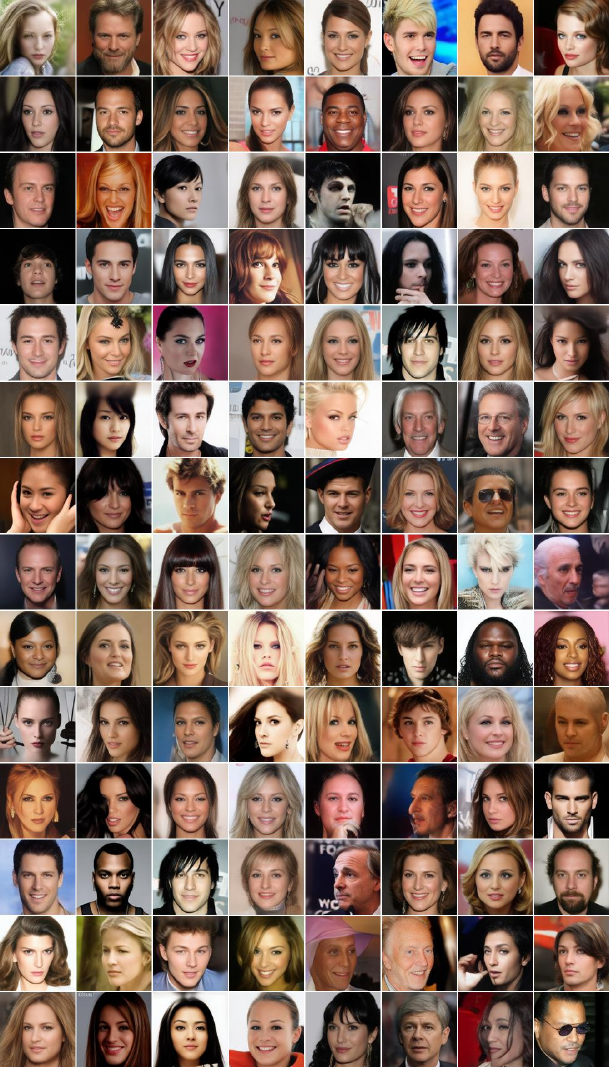}
  \caption{Unconditional CelebA-HQ $256{\times}256$ samples from RS-IMLE + RTM.}
  \label{fig:showcase_celeba_p1}
\end{figure}
\clearpage

\begin{figure}[!t]
  \centering
  \includegraphics[width=\linewidth,height=0.92\textheight,keepaspectratio]{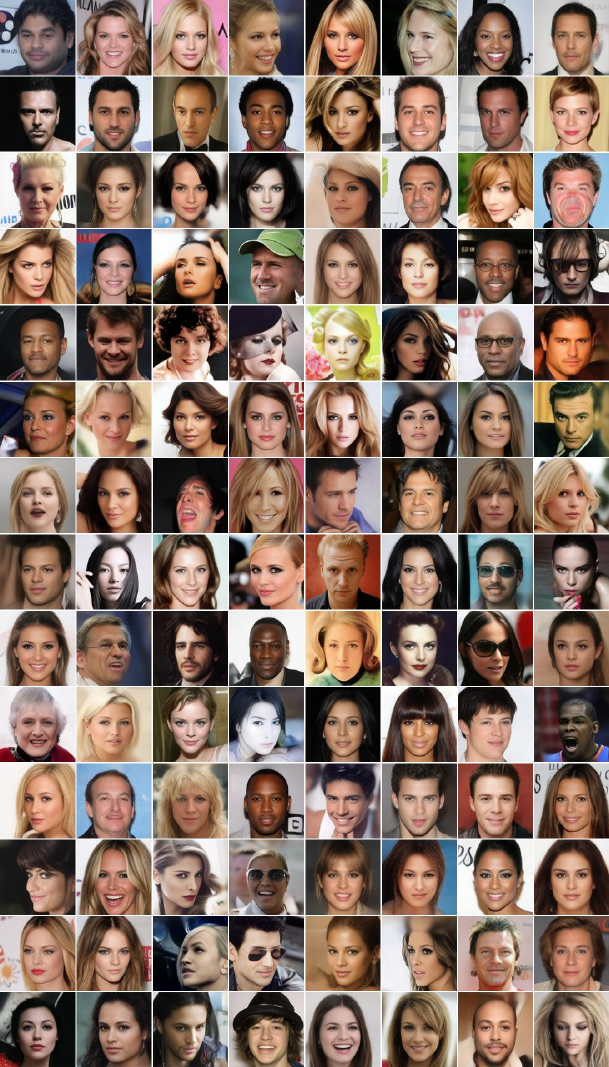}
  \caption{Unconditional CelebA-HQ $256{\times}256$ samples from RS-IMLE + RTM.}
  \label{fig:showcase_celeba_p2}
\end{figure}
\clearpage

\begin{figure}[h]
  \centering
  \includegraphics[width=0.92\linewidth]{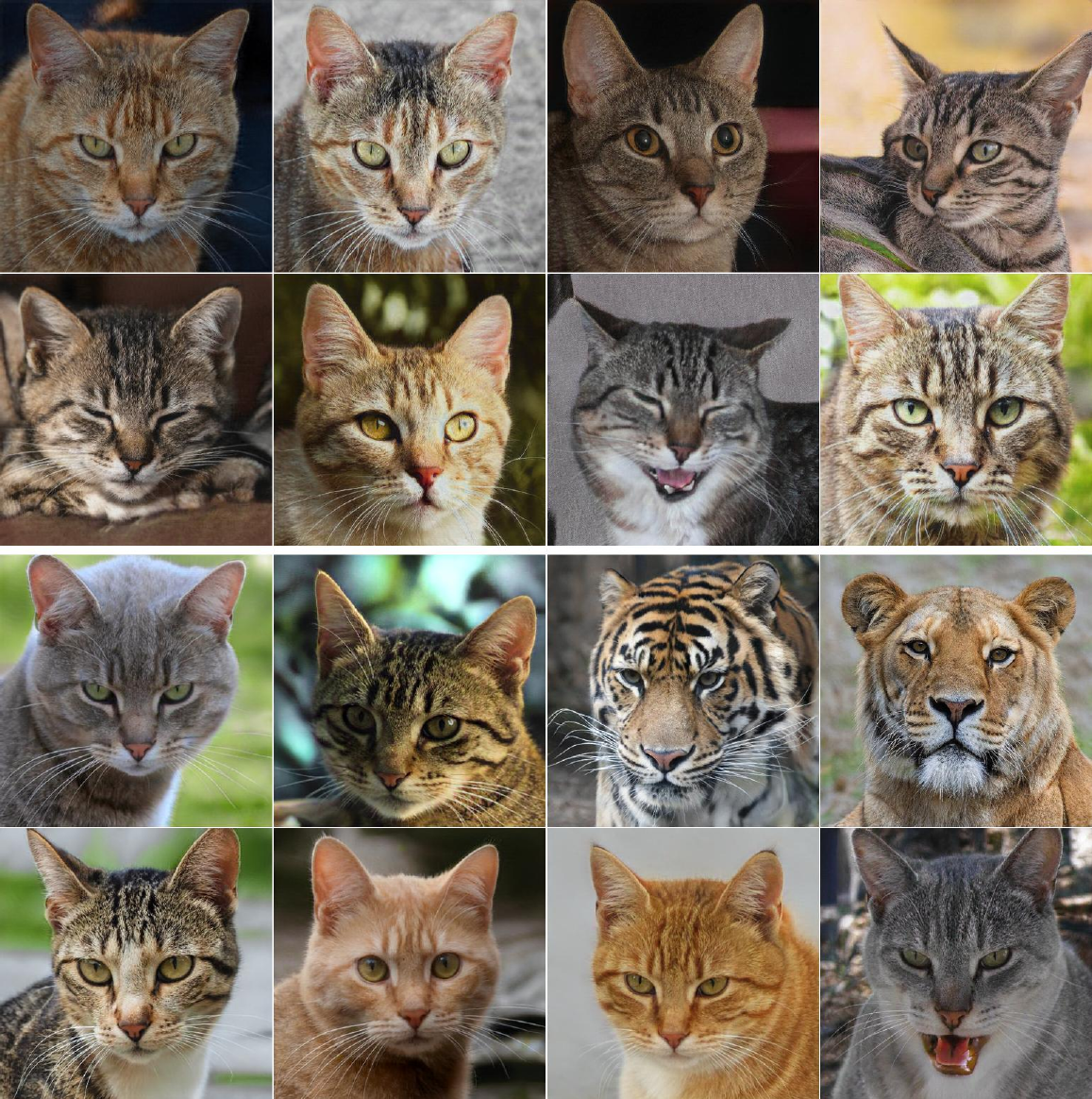}
  \caption{AFHQ-v1 with StyleGAN2-ADA. Top: baseline. Bottom: with our RTM mapper. First set of examples.}
  \label{fig:showcase_afhq_p1}
\end{figure}

\begin{figure}[h]
  \centering
  \includegraphics[width=0.92\linewidth]{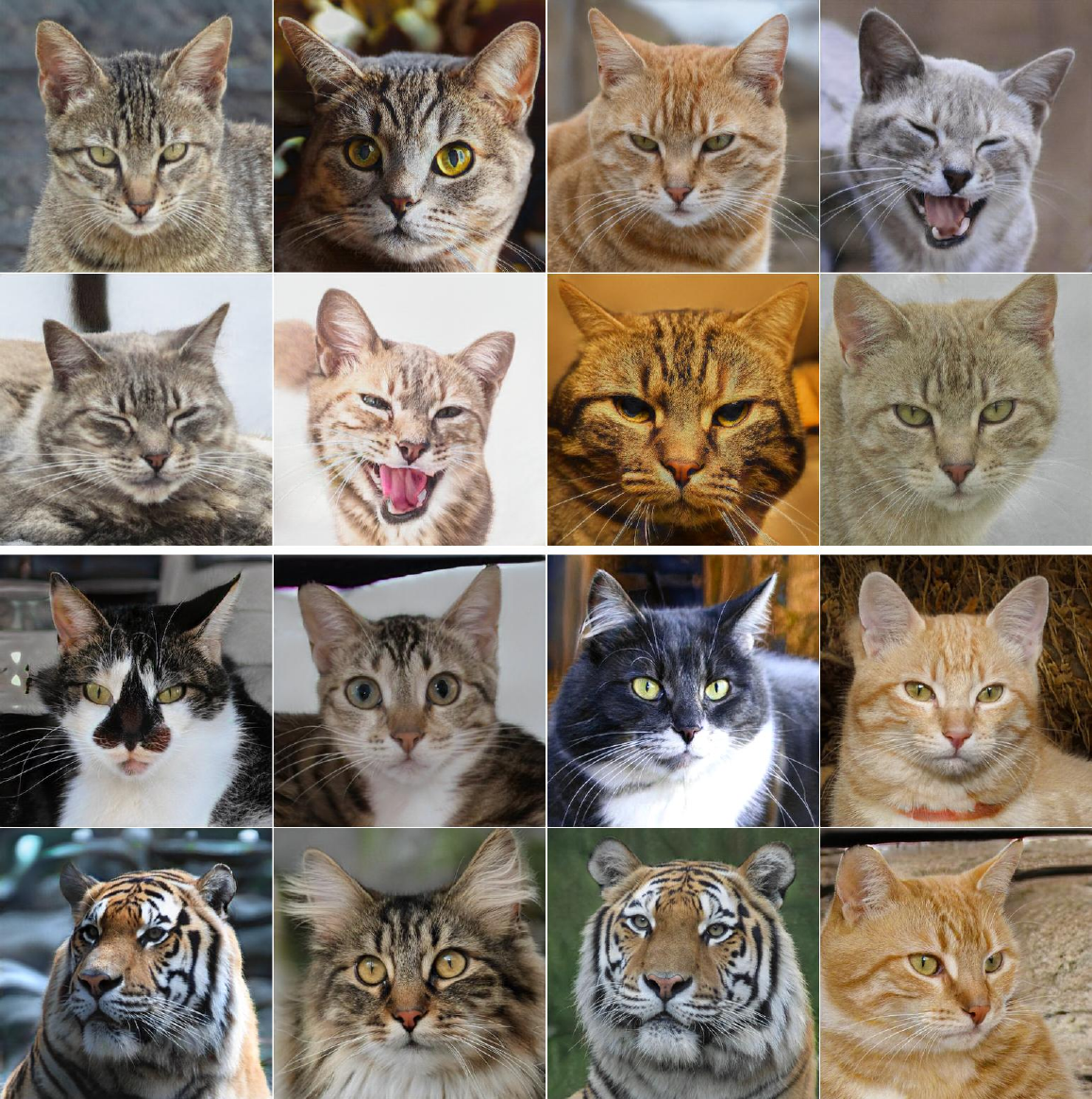}
  \caption{AFHQ-v1 with StyleGAN2-ADA. Top: baseline. Bottom: with our RTM mapper. Second set of examples.}
  \label{fig:showcase_afhq_p2}
\end{figure}

\clearpage


\end{document}